\documentclass[11pt]{article}
\usepackage[margin=1in]{geometry}

\usepackage[makeroom]{cancel}
\usepackage[normalem]{ulem}
\usepackage[utf8]{inputenc} 
\usepackage[T1]{fontenc}    
\usepackage{hyperref}       
\usepackage{url}            
\usepackage{minitoc}
\usepackage{booktabs}       
\usepackage{amsfonts}       
\usepackage{nicefrac}       
\usepackage{microtype}      
\usepackage{xcolor}         
\usepackage[toc,page,header]{appendix}
\usepackage[export]{adjustbox}
\usepackage{wrapfig}
\usepackage{microtype,bm}
\usepackage{graphicx}
\usepackage{booktabs} 
\usepackage{adjustbox}
\usepackage{hyperref}
\usepackage{caption}
\usepackage{braket}
\usepackage{subcaption}

\usepackage{amsmath}
\usepackage{amssymb}
\usepackage{mathtools}
\usepackage{amsthm}

\usepackage{caption}
\usepackage{algorithm}
\usepackage{algpseudocode}

\usepackage[capitalize,noabbrev]{cleveref}

\theoremstyle{plain}

\usepackage[textsize=tiny]{todonotes}

\usepackage{color}
\usepackage[utf8]{inputenc} 
\usepackage[T1]{fontenc}    
\usepackage{url}            
\usepackage{booktabs}       

\usepackage{xcolor,colortbl}

\usepackage{amsfonts}       
\usepackage{nicefrac}       
\usepackage{microtype}      
\usepackage{booktabs}
\usepackage[rightcaption]{sidecap}
\sidecaptionvpos{figure}{c}
\usepackage{bbm}
\usepackage{latexsym, epsfig, amssymb, amsmath, amsthm, mathrsfs, booktabs,mathtools}
\usepackage{enumerate}
\usepackage{multirow}
\usepackage{multicol}
\usepackage{array,booktabs}
\usepackage{enumitem}
\usepackage{epstopdf}
\usepackage{subcaption}
\usepackage{wrapfig}
\usepackage{adjustbox}
\usepackage{multirow}
\usepackage{multicol}
\usepackage{array,booktabs}
\usepackage{enumitem}
\usepackage{epstopdf}
\usepackage{xspace}
\usepackage{comment}
\usepackage{multirow}
\usepackage{makecell}

\usepackage{enumitem,amssymb}
\usepackage{pifont}
\usepackage{tikz,lipsum,lmodern}
\usepackage[most]{tcolorbox}

\newenvironment{myitemize}{\begin{list}{$\bullet$}
		{\setlength{\topsep}{1mm}
			\setlength{\itemsep}{0.25mm}
			\setlength{\parsep}{0.25mm}
			\setlength{\itemindent}{0mm}
			\setlength{\partopsep}{0mm}
			\setlength{\labelwidth}{15mm}
			\setlength{\leftmargin}{4mm}}}{\end{list}}

\newcommand{\mc}{\mathcal}
\newcommand{\E}{\mb{E}}
\newcommand{\m}[1]{{\bm{#1}}}

\renewcommand{\mc}[1]{\ensuremath{\mathcal{#1}}} 
\newcommand{\g}[1]{\mbox{\boldmath $#1$}}
\newcommand{\mb}[1]{{\mathbb{#1}}}

\DeclarePairedDelimiterX{\norm}[1]{\lVert}{\rVert}{#1}

\definecolor{darkred}{RGB}{150,0,0}
\definecolor{darkgreen}{RGB}{0,150,0}
\definecolor{darkblue}{RGB}{0,0,200}
\hypersetup{colorlinks=true, linkcolor=darkred, citecolor=darkgreen, urlcolor=darkblue}

\let\phi\varphi

\newtheorem{theorem}{Theorem}[section] 
\newtheorem{lemma}{Lemma}[section] 
\newtheorem{corollary}{Corollary}[section] 
\newtheorem{definition}{Definition}[section] 
\newtheorem{assumption}{Assumption}

\newcommand{\nn}{\nonumber}

\newcommand{\fedavg}{\text{FedAvg}\xspace}

\newcommand{\scaffold}{\text{SCAFFOLD}\xspace}

\newcommand{\fedmsa}{\text{FedMSA}\xspace}
\newcommand{\fednest}{\text{FedNest}\xspace}

\definecolor{Gray}{gray}{0.85}
\definecolor{LightCyan}{rgb}{0.88,1,1}

\newcolumntype{a}{>{\columncolor{Gray}}c}
\newcolumntype{b}{>{\columncolor{white}}c}

\usepackage{amssymb}
\usepackage{pifont}

\newtcolorbox{mybox}[3][]
{
  colframe = #2!15,
  colback  = #2!10,
  coltitle = #2!10!black,  
  title    = {#3},
  boxsep   = 0.25pt,
  left     = 0.5pt,
  right    = 0.5pt,
  top      = 0pt,
  bottom   = 0pt,
  width=\linewidth+13pt,
  #1,
}

\title{Federated Multi-Sequence Stochastic Approximation with\\Local Hypergradient Estimation}

\author{{Davoud Ataee Tarzanagh}\thanks{University of Michigan and University of Pennsylvania, Email: \texttt{tarzanaq@upenn.edu}}
\and{Mingchen Li }\thanks{University of California, Riverside, Email: \texttt{mli176@.ucr.edu}
}
\and{Pranay Sharma}\thanks{Carnegie Mellon University, Email: \texttt{pranaysh@andrew.cmu.edu}
}
\and{{Samet Oymak}\thanks{University of Michigan and University of California, Riverside, Email: \texttt{oymak@umich.edu}}
}
}

\date{}
\begin{document}
\doparttoc 
\faketableofcontents 

\maketitle

\begin{abstract}
Stochastic approximation with multiple coupled sequences (MSA) has found broad applications in machine learning as it encompasses a rich class of problems including bilevel optimization (BLO), multi-level compositional optimization (MCO), and reinforcement learning (specifically, actor-critic methods). However, designing provably-efficient federated algorithms for MSA has been an elusive question even for the special case of double sequence approximation (DSA). Towards this goal, we develop FedMSA which is the first federated algorithm for MSA, and establish its near-optimal communication complexity. As core novelties, (i) FedMSA enables the provable estimation of hypergradients in BLO and MCO via local client updates, which has been a notable bottleneck in prior theory, and (ii) our convergence guarantees are sensitive to the heterogeneity-level of the problem. We also incorporate momentum and variance reduction techniques to achieve further acceleration leading to near-optimal rates. Finally, we provide experiments that support our theory and demonstrate the empirical benefits of FedMSA. As an example, FedMSA enables order-of-magnitude savings in communication rounds compared to prior federated BLO schemes. Code is available at \url{https://github.com/ucr-optml/FedMSA}.
\end{abstract}
\section{Introduction}\label{sec:fedavg:full}
Stochastic approximation (SA) methods \cite{robbins1951stochastic} are iterative techniques widely used in machine learning (ML) to estimate zeros of functions when only noisy function value estimates are available. Initially, SA focused on asymptotic convergence for simple problems, such as finding solutions to $\m{g}(\m{x}) = 0$ or minimizing $f(\m{x})$. However, recent years have seen increased interest in more complex applications, including bilevel and multi-level stochastic optimization problems, leading to the development of double-sequence \cite{borkar1997stochastic}  and multi-sequence SA \cite{shen2022single} techniques to address these challenges. For example, the bilevel problem \eqref{main:blo} can be effectively tackled using double-sequence stochastic approximation (DSA). By imposing appropriate smoothness conditions, such as the strong convexity of $g$ and the differentiability of $f$ and $g$, we are able to derive the first-order optimality conditions: if $(\m x, \m w)$ is a local minimum of \eqref{main:blo}, there exists a unique $\m{v}$ such that
\begin{equation}\label{eqn:optima}
    \left.\begin{matrix}
    \hspace{-77pt} & \nabla_{\m x} f(\m x, \m w) + \nabla^2_{\m x, \m w} g(\m x, \m w) \m v= \m 0, \vspace{4pt}\\ 
& \nabla_{\m w}^2 g(\m x, \m w)   \m v + \nabla_{\m w} f(\m x, \m w)=\m 0,\\
&\nabla_{\m w} g(\m x, \m w) = \m 0.
    \end{matrix}\right\}
\end{equation}
\tcbset{colback=white!5!white,colframe=violet!75!black,fonttitle=\bfseries}
\begin{tcolorbox}[title=\qquad \qquad \qquad \qquad \qquad BLO  and its Mappings,height=4cm, sidebyside,righthand width=7cm]
\begin{align}\tag{BLO}\label{main:blo}
&\underset{\m{x} \in \mb{R}^{{d}_0}}{\min}  
 f\left(\m{x},\m{w}^\star(\m{x})\right), 
\\
\nonumber
\hspace{-.8cm}&\text{s.t.}~~
\m{w}^\star(\m{\m{x}})
\in \underset{ \m{w}\in \mb{R}^{{d}_1}}{\textnormal{argmin}}~g\left(\m{x},\m{w}\right) 
\end{align}
\tcblower
\begin{align}\label{eq:map:bl}
\nonumber
\hspace{-.3cm}
\mb{S}(\m{x},\m{z})&=
\begin{bmatrix}
 \nabla_\m{w} g(\m{x}, \m{w})
\\
    \nabla_{\m{w}}^2 g(\m{x}, \m{w}) \m{v} - \nabla_{\m{w}} f(\m{x},\m{w})
\end{bmatrix}    
\nonumber \\
\hspace{-.5cm}
\nonumber
\mb{P}(\m{x},\m{z}) &= \nabla_\m{x} f(\m{x},\m{w}) -\nabla_{\m{xw}}^2 g(\m{x},\m{w})\m{v},  \\
\text{with}~~\m{z}&=[\m{w},\m{v}].
\end{align}
\end{tcolorbox}
\vspace{-.48cm}
\tcbset{colback=white!5!white,colframe=violet!75!black,fonttitle=\bfseries}
\begin{tcolorbox}[title= \qquad \qquad \qquad \qquad \qquad MCO and its Mappings,height=3.5cm, sidebyside,righthand width=6.4cm]
\begin{small}
\begin{align*}\tag{MCO}\label{fedmco_main}
\hspace{-.4cm}\min_{\m{x} \in \mb{R}^{{d}_0}} f^{N}\circ f^{N-1}\circ \cdots \circ  f^{0}(\m{x})
\end{align*}
\end{small}
\tcblower
\begin{small}
\begin{align}\label{eq:map:mcp}
\nonumber
&\mb{S}(\m{z}^{n-1},\m{z}^n)=
\m{z}^{n}- f^{n-1} (\m{z}^{n-1}),~~\forall n \in [N] \\
\nonumber
&\mb{P}(\m{z}^0, \ldots, \m{z}^N) =  \nabla f^{0} (\m{z}^0) \cdots  \nabla f^{N}(\m{z}^N)\\
 &\text{with}~~\m{z}^0=\m{x}.
\end{align}
\end{small}
\end{tcolorbox}
Notably,  the optimality conditions \eqref{eqn:optima} can be reformulated as solving a system of nonlinear equations, which can be effectively addressed using DSA. The framework of double mappings, as defined in \eqref{eq:map:bl}, is employed to facilitate this process, as discussed in \cite{dagreou2022framework}. The DSA approach can be extended to handle multiple ($N \geq 2$) sequences that are employed in solving problems related to multi-level stochastic optimization \cite{sato2021gradient,yang2019multilevel,balasubramanian2022stochastic}; see \eqref{fedmco_main}. Specifically, \cite{shen2022single} introduced an extension of the DSA to find $\m{x}$, $\m{z}^{1}$, $\ldots$, $\m{z}^{N}$ such that
\begin{equation}\tag{MSA}
\label{msa:prob}
    \left.\begin{matrix}
    {\mb{P}} \left(\m{x}, \m{z}^{1}, \ldots, \m{z}^{N} \right)=\m{0}, \vspace{4pt}\\ 
    {\mb{S}}^{n} \left(\m{z}^{n-1}, \m{z}^{n} \right) = \m{0}, \forall \ n \in [N].
    \end{matrix}\right\}
\end{equation}
Here, ${\mb{P}}: \mb R^{d_0 \times d_1 \times \dots \times d_N} \to \mb R^{d_0}$ and ${\mb{S}}^{n}: \mb R^{d_{n-1} \times d_n} \to \mb R^{d_{n-1}}$, for all $n \in [N]$. 
Note that the outer-level mapping $\mb{P}$ depends on all the variables, while each of ${\mb{S}}^{n}$ depends only on $(\m{z}^{n-1}, \m{z}^{n})$, for all $n \in [N]$.

In various modern applications, the data required to compute noisy samples of $\mb{P}$ and $\{\mb{S}^{n}\}_{n}$ is naturally distributed across multiple nodes or clients. Due to privacy concerns of individual clients or communication limitations within the system, it is not feasible to collect the data in a centralized location for computation. Federated Learning (FL) is a paradigm designed to address this challenge by decentralizing the computation to individual clients. In this study, we tackle the problem defined in \eqref{msa:prob} within a federated setting. We consider a network comprising $M$ clients, each possessing their own local mappings $\mb{P}^m$ and $\{\mb{S}^{m,n}\}_{n}$, where $m \in [M]$. The objective is to find the optimal values of $\m{x}$, $\m{z}^{1,}$, $\ldots$, $\m{z}^{N,}$ such that
\begin{equation}\tag{Fed-MSA}
\label{fedmsa:prob}
    \left.\begin{matrix}
    \hspace{-77pt}\sum_{m=1}^M {\mb{P}}^{m}\left(\m{x},\m{z}^{1}, \ldots, \m{z}^{N}\right)=\m{0}, \vspace{4pt}\\ 
    \sum_{m=1}^M {\mb{S}}^{m,n}\left(\m{z}^{n-1},\m{z}^{n}\right)=\m{0}, \forall \ n \in [N].
    \end{matrix}\right\}
\end{equation}
Clearly, $\mb{P} := \sum_{m=1}^M \mb{P}^m$ and $\mb{S}^n := \sum_{m=1}^M \mb{S}^{m,n}$, for all $n \in [N]$.

\paragraph{Challenge of Local Updates in \eqref{fedmsa:prob}.} 
FL systems involve clients performing multiple local steps using their local data between synchronization rounds to minimize communication costs. The analysis of the single-level SA in this context has been extensively explored \cite{khaled2019first}. However, the presence of multiple sequences and their local updates poses a significant challenge for more complex SA methods. To understand this challenge, let's consider the simplified bilevel problem \eqref{main:blo} and its optimality conditions \eqref{eqn:optima}. Computing the local hypergradient (i.e., mapping $\mb{P}^m$) requires calculating the global Hessian, while each client $m$ only has access to their local Hessian. Existing approaches address this challenge by maintaining a fixed global Hessian during local iterations, resulting in an inexact local hypergradient \cite{tarzanagh2022fednest,xiao2022alternating,huang2022fast,huang2023achieving,xiao2023communication}; please refer to Sec.~\ref{sec:fedblo} for further discussions. 
\begin{figure}
  \centering
  \begin{subfigure}{0.33\textwidth}
    \centering
    \begin{tikzpicture}
      \node at (0,0) {\includegraphics[scale=0.33]{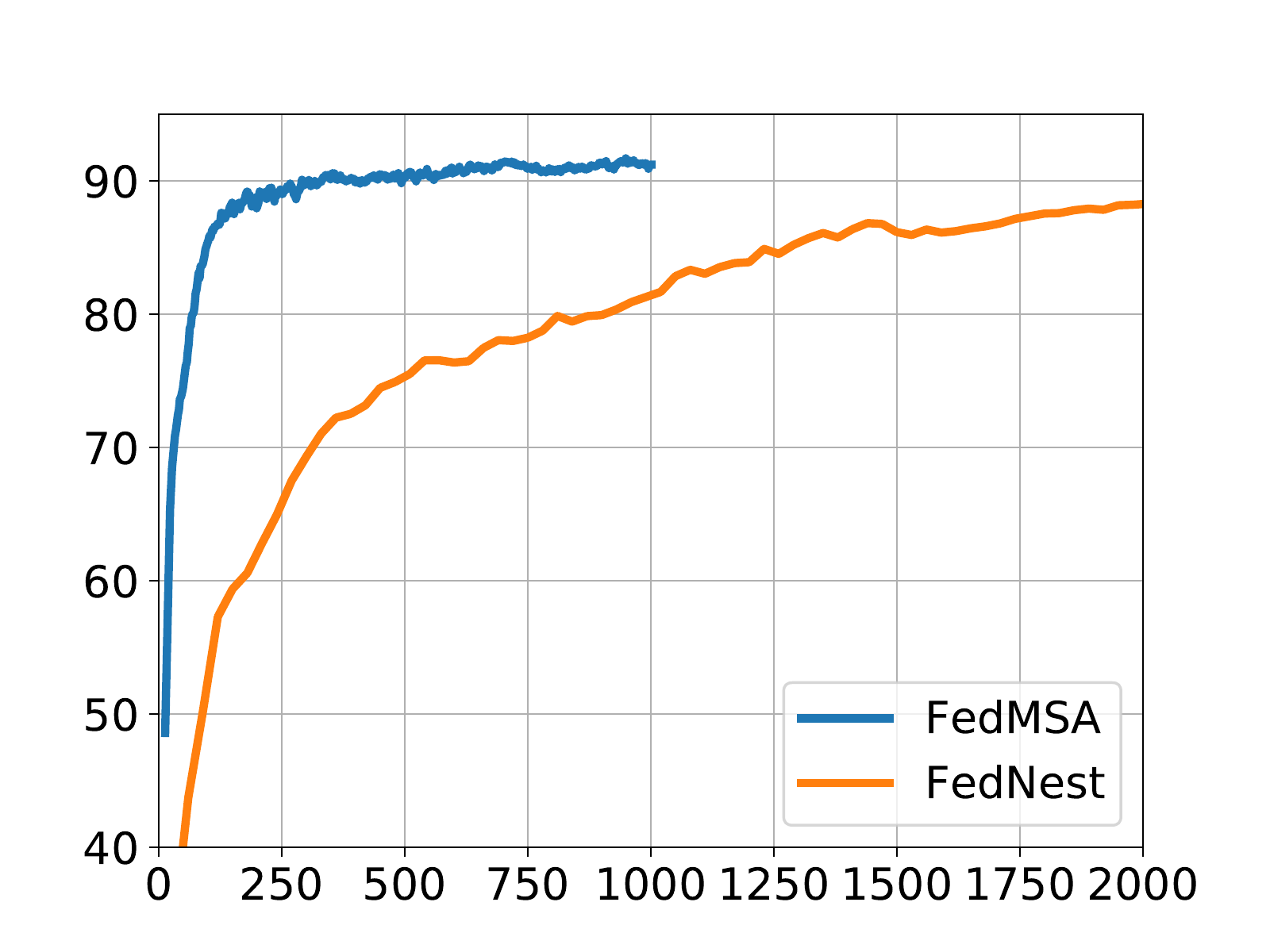}};
      \node[font=\small, scale=0.8] at (0,-2.1) {Communication Rounds ($R$)};
      \node[font=\small, scale=0.8, rotate=90] at (-2.6,-0.2) {Test Accuracy};
    \end{tikzpicture}
    \caption{}
    \label{fig:intro_subfig_less_comm}
  \end{subfigure}\hfill\begin{subfigure}{0.33\textwidth}
    \centering
    \begin{tikzpicture}
      \node at (0,0) {\includegraphics[scale=0.33]{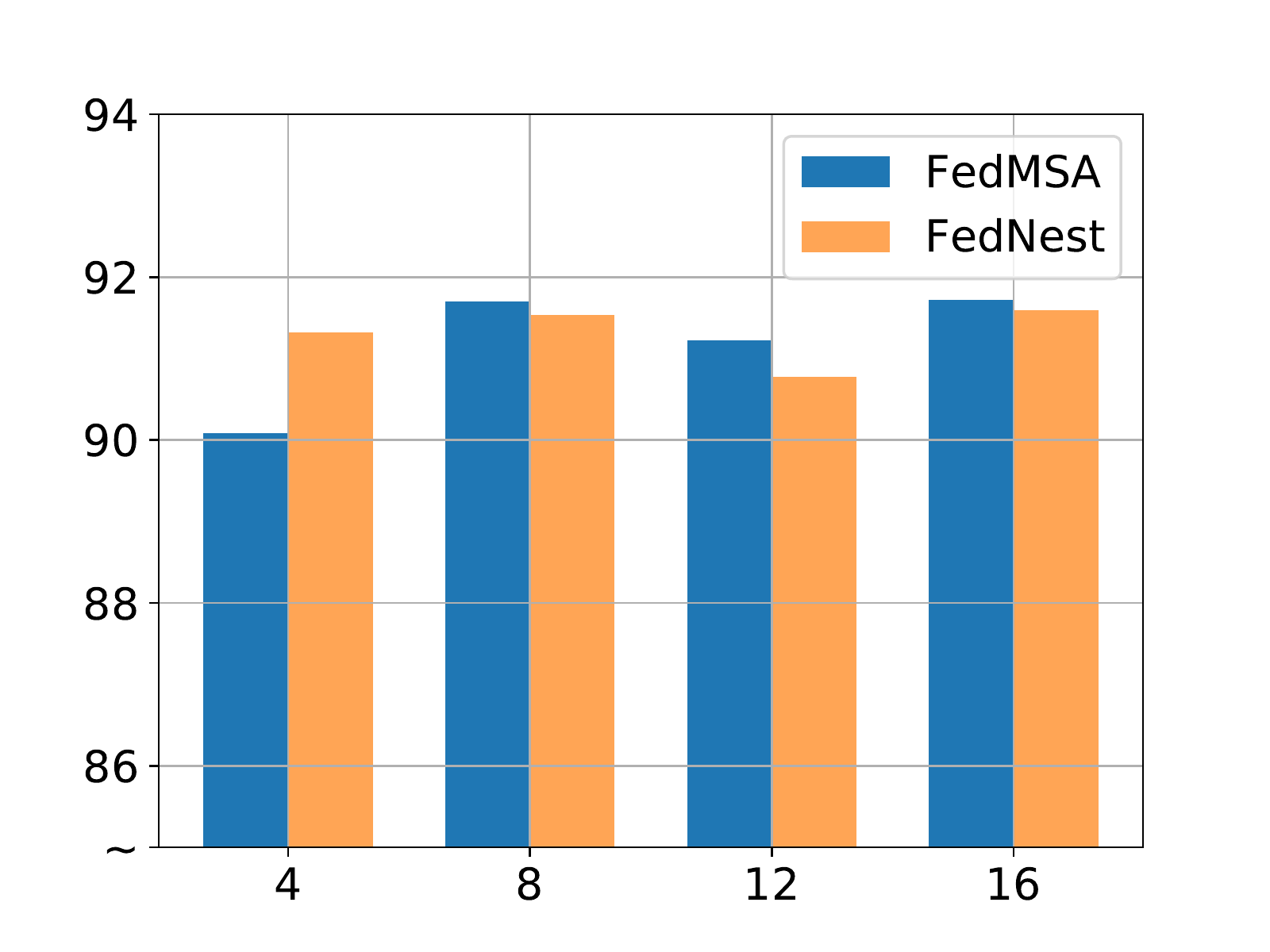}};
      \node[font=\small, scale=0.8] at (0,-2.1) {Local Updates ($K$)}; 
     \node[font=\small, scale=0.8, rotate=90] at (-2.6,-0.2) {Test Accuracy};
    \end{tikzpicture}
    \caption{ }
    \label{fig:intro_subfig_tau}
  \end{subfigure}\begin{subfigure}{0.32\textwidth}
    \centering
    \begin{tikzpicture}
      \node[inner sep=0pt, anchor=north] (list) at (0,0) {
        \begin{minipage}[t][4.9cm][t]{0.9\textwidth} 
        \textbf{FedMSA's Key Features:}
          \begin{itemize}[label=\textcolor{green!70!black}{\ding{51}},leftmargin=*]
            \item Local hypergradient estimation (provably)
            \item Order-of-magnitude savings in communication rounds
            \item Faster convergence
          \end{itemize}
        \end{minipage}
      };
    \end{tikzpicture}
    \label{fig:intro_subfig_benefits}
  \end{subfigure}
  \caption{Loss function tuning on an imbalanced dataset \cite{li2021autobalance} (details in Sec.~\ref{sec:numerics}). Our findings are as follows: \textbf{(a)} FedMSA achieves a significant reduction of 10x in communication rounds compared to FedNest \cite{tarzanagh2022fednest}. \textbf{(b)} Besides enabling much faster convergence, FedMSA also enjoys higher eventual accuracy: When the number of local updates is large, FedMSA successfully updates the indirect component of hypergradient in local iterations, whereas FedNest fails to update the indirect hypergradient, resulting in inferior performance.
  }
  \label{fig:intro}
\end{figure}
\paragraph{Contributions:} In this work, we address these fundamental challenges surrounding federated MSA through the following innovations.
\begin{myitemize}
\item \textbf{Federated Local Mapping and Hypergradient Estimation:}
Our novel strategy enables the federated estimation of local mappings through local iterations. In contrast to previous approaches on bilevel \cite{tarzanagh2022fednest,xiao2022alternating,huang2022fast,huang2023achieving,xiao2023communication} and compositional \cite{tarzanagh2022fednest,huang2022faster} optimization, our method successfully updates the indirect component of the hypergradient within the local iteration of FL resulting in significant benefits with an increased value of $K$; see Figure~\ref{fig:intro_subfig_tau}.
\item \textbf{A New Algorithm: \fedmsa.} We introduce \fedmsa, a novel federated stochastic approximation algorithm with near-optimal communication complexity. The convergence guarantees of our algorithm depend on the level of problem heterogeneity and offer significant speedup in convergence. By integrating momentum and client-drift/variance reduction techniques, \fedmsa achieves state-of-the-art convergence rates, even for standard non-federated problems; see Theorem~\ref{thm:fedmsa}.

\item \textbf{Bilevel Optimization (Sec \ref{sec:fedblo}).} 
In addition to achieving faster convergence rates, our approach addresses limitations observed in previous works on BLO during the update of the inner and outer variables ($\m{w}$ and $\m{x}$) \cite{tarzanagh2022fednest,xiao2022alternating,huang2022fast,huang2023achieving,xiao2023communication}. Specifically, \fedmsa performs simultaneous updates of the inner problem solution, the linear system, and the outer variable, resulting in improved communication efficiency; see Figure~\ref{fig:intro_subfig_less_comm}.

\item \textbf{Compositional Optimization (Sec \ref{section:sco}).} 
Existing federated methods for MCO focus solely on the double sequence case ($N=1$) \cite{huang2021compositional,tarzanagh2022fednest,huang2022faster}. In contrast, our approach extends these results to the multi-level scenario, offering enhanced communication complexity analysis for arbitrary values of $N$; see Table~\ref{table:fedmsa:results}.
 \end{myitemize}

\begin{table}[t]
\centering
\begin{tabular}{l l|c|c| }
  \multicolumn{4}{c}{} \\
 \cline{3-4}
&{} & \textbf{Sample } &\textbf{Communication}  \\
 & & \textbf{Complexity}   &  \textbf{Complexity}\\  
 \cline{3-4}
\\
 \cline{1-4}
  \multicolumn{1}{|c}{\textbf{MSA}}&\multicolumn{1}{|l|}{{STSA}\cite{shen2022single} }& \multicolumn{1}{c|}{ $ \epsilon^{-2}$}
 &  N.A.\\
 \cline{2-4}
\multicolumn{1}{|c}{ }& \multicolumn{1}{|l|}{{FedMSA} }& 
$ \epsilon^{-1.5}$ 
& $  \tau\epsilon^{-1}$ \\
 \cline{1-4}
 \\
 \cline{3-4}
 \cline{1-4}
\multicolumn{1}{|c}{ }&\multicolumn{1}{|l|}{{FedNest} \cite{tarzanagh2022fednest}}& $ \epsilon^{-2} $  & $\epsilon^{-2}$
\\
 \cline{2-4}
\multicolumn{1}{|c}{\textbf{BLO}} &\multicolumn{1}{|l|}{{FedMBO} \cite{huang2023achieving}}& $ \epsilon^{-2} $ &   $ \epsilon^{-2}$
\\
 \cline{2-4}
\multicolumn{1}{|c}{} &\multicolumn{1}{|l|}{{AdaFBiO} \cite{huang2022fast} }&$  \epsilon^{-1.5} $  &  $ \epsilon^{-1.5}$
\\
 \cline{2-4}
\multicolumn{1}{|c}{}&\multicolumn{1}{|l|}{{FedMSA} }& $ \epsilon^{-1.5} $  &   $ {\tau}\epsilon^{-1}$  \\
 \cline{1-4}
 \\
\cline{1-4}
\multicolumn{1}{|c}{  }&\multicolumn{1}{|l|}{{ComFedL} ($N= 1$)  \cite{huang2021compositional} }&  $  \epsilon^{-2} $ &  $\epsilon^{-2}$ \\
 \cline{2-4}
 \cline{2-4}
\multicolumn{1}{|c}{\textbf{MCO}}& \multicolumn{1}{|l|}{{FedNest} ($N= 1$) \cite{tarzanagh2022fednest}}& $  \epsilon^{-2}$ &  $ \epsilon^{-2}$ \\
 \cline{2-4}
\multicolumn{1}{|c}{ }&\multicolumn{1}{|l|}{{AdaMFCGD} ($N= 1$) \cite{huang2022faster} }&   $  \epsilon^{-1.5}$ & $ \epsilon^{-1.5}$
\\
 \cline{2-4}
\multicolumn{1}{|c}{ }&\multicolumn{1}{|l|}{{FedMSA} ($N\geq 1$)}& $  \epsilon^{-1.5}$  &   $ {\tau}\epsilon^{-1}$ 
\\
  \cline{1-4}
 \end{tabular}
\vspace{0.2cm}
\caption{Comparison of {sample complexity} and {communication complexity} among various algorithms to achieve an $\epsilon$-stationary point of MSA, BLO and MCO problems. Here,   $\tau$ factor controls the benefit we can obtain from the small heterogeneity of the mappings. Discussions are provided below Theorem~\ref{thm:fedmsa}.
}
  \label{table:fedmsa:results}
\end{table}
 


\section{Our Setting and Proposed Algorithm}\label{sec:fedavg:full}
In this section, we first introduce some noations and definitions that will be used in our analysis.
$\mb{N}$ and $\mb{R}$ denotes the set of natural and real numbers, respectively. We consider distributed optimization over $M$ clients and we denote $[M]:=\{1,\ldots,M\}$. For vectors $\m{v} \in \mb{R}^d$ and matrix $\m{M} \in \mb{R}^{d \times d}$, we denote $\|\m{v}\|$ and $\|\m{M}\|$ the respective Euclidean and spectral norms.  Following the literature on single-level stochastic \cite{fang2018spider,lei2017non,zhou2020stochastic} and federated \cite{murata2021bias,karimireddy2020mime,patel2022towards} gradient-based methods for finding a stationary point of an optimization problem, we consider stochastic optimization problems that access  $ ({\mb{P}}^{m}, \{ {\mb{S}}^{m,n} \}_{n})$ via
\begin{equation}\label{eqn:stoch:oracle}  
    \begin{split}
    {\mb{P}}^{m}\left(\m{x}, \m{Z}\right) &:= \mb{E}_{\xi \sim \mc{A}^m} \left[\m{p}^{m}\left(\m{x},\m{Z}; \xi\right)\right],\\
    {\mb{S}}^{m,n}(\m{z}^{n-1},\m{z}^{n}) &:= \mb{E}_{\zeta^n \sim \mc{B}^{m,n}} \left[
    \m{s}^{m,n}(\m{z}^{n-1},\m{z}^{n}; \zeta^n)\right], \forall n \in [N].
    \end{split}   
\end{equation}
Here, $\m{Z}=[\m{z}^1, \ldots, \m{z}^N]$, $(\xi, \{\zeta^n\}_{n}) \sim (\mc{A}^m, \{\mc{B}^{m,n}\}_{n})$ denote the stochastic samples for the $m^{\text{th}}$ client. 
\begin{definition}
    The mappings $\m{p}(\m{w})$ and $\m{p}(\m{w}; \xi)$  are called $L$-Lipschitz and  $\bar{L}$--mean Lipschitz if for any $\m{w},\bar{\m{w}}\in\mathbb{R}^d$, $\Vert \m{p}(\m{w})- \m{p}(\bar{\m{w}}) \Vert \leq L \Vert \m{w}-\bar{\m{w}} \Vert$ and  $ \mb{E}_{\xi \sim\mc{D}}\Vert \m{p}(\m{w};\xi)- \m{p}(\bar{\m{w}}; \xi) \Vert \leq \bar{L} \Vert \m{w}-\bar{\m{w}} \Vert$, respectively. 
\label{def:mean:lip}
\end{definition}
\begin{definition}\label{def:heter}
 $\{\m{p}^{m}(\m{w})\}_m$ are called $\upsilon$--Heterogeneous   if 
    $ \sup_{m \in [M], \m{w}}\|\nabla\m{p}^{m}(\m{w})-\nabla \m{p}(\m{w})\| \leq 	\upsilon$  for  some $\upsilon$. Here, $\m{p} = \sum_{m=1}^M \m{p}^m$.
\end{definition}

\begin{algorithm}[t]
\caption{ Federated Multi-Sequence Stochastic Approximation (\fedmsa) 
}
 \label{alg:fedmsa}
\begin{algorithmic}[1]
\State \textbf{Input}: Initialization $(\m{x}_0,\m{Z}_0)$, the number of communication rounds $R$, 
step-sizes $(\alpha, \beta_1, \ldots, \beta_N)$, momentum parameter $1\geq \rho \geq 0$. 
\State $(\m{x}_{-1},\m{Z}_{-1})=(\m{x}_0,\m{Z}_0)$.
\For{$r=0, \ldots, R-1$} \hfill 
 \State \textbf{if} $r=0$ set $\rho=1$ and $K=1$.
 \For {$m \in [M]$ \textbf{in parallel}}  
    \State Compute  $(\m{h}^m_r,\m{q}^{m,1}_r, \ldots, \m{q}^{m,N}_r)$ according to Eq. \eqref{eqn:moment:maps}.
\EndFor
  \State $ (\m{h}_r, \m{q}^1_r, \ldots, \m{q}^N_r) = 1/M \sum_{m=1}^M (\m{h}^m_r,  \m{q}^{m,1}_r, \ldots, \m{q}^{m,N}_r)$.
  \State  Communicate $(\m{h}_r, \m{q}^1_r, \ldots, \m{q}^N_r)$ to client $\tilde{m}$, where $ \tilde{m} \sim \textnormal{Unif}~ [M]$
\begin{tcolorbox}[title= Local MSA with momentum-type variance reduction on client $\tilde{m}$, colframe=blue!5,colback=blue!5,colbacktitle=blue!45, boxsep= 2pt,left=1pt,right=0.0pt,height=4.5cm, width=12.5cm]
\State  $(\m{h}_{r,0}^{\tilde{m}}, \m{q}^{\tilde{m},1}_{r,0}, \ldots, \m{q}^{\tilde{m},N}_{r,0}) :=(\m{h}_r, \m{q}^1_r, \ldots, \m{q}^N_r)$
\State $\m{x}_{r+1,1}^{\tilde{m}}:= \m{x}_{r+1,0}^{\tilde{m}}:= \m{x}_{r}$, $\m{Z}_{r+1,1}^{\tilde{m}}:= \m{Z}_{r+1,0}^{\tilde{m}}:= \m{Z}_{r}$
\For{$k=1,\ldots,K$}
\State Compute  $(\m{h}^{\tilde{m}}_{r,k},\m{q}^{{\tilde{m}},1}_{r,k}, \ldots, \m{q}^{{\tilde{m}},N}_{r,k})$ according to Eq. \eqref{eqn:moment:local:maps}. 
\State
$\m{x}_{r+1, k +1}^{\tilde{m}} = \m{x}_{r+1,k}^{\tilde{m} }- \alpha \m{h}_{r,k}^{\tilde{m}}$ 
\State  $\m{z}_{r+1,k+1}^{\tilde{m},n} = \m{z}_{r+1,k}^{\tilde{m},n} - \beta_{n} \m{q}_{r,k}^{\tilde{m},n}, ~~ n =1,2,...,N.$ 
\EndFor
\end{tcolorbox}  
\State $(\m{x}_{r+1}, \m{Z}_{r+1})= (\m{x}_{r+1,K+1}^{\tilde{m}},\m{Z}_{r+1,K+1}^{\tilde{m}})$
\EndFor
\State \textbf{Output}:  $(\tilde{\m{x}},\tilde{\m{Z}}) = (\m{x}_{r,k}^{\tilde{m}},\m{Z}_{r,k}^{\tilde{m}})$,~ $(r,k) \sim \textnormal{Unif}~ [R]\times[K]$ 
\end{algorithmic}
\end{algorithm}

\subsection{The Proposed Framework: \fedmsa}

Our proposed approach to federated multi-sequence approximation (abbreviated as \fedmsa) is described in Algorithm~\ref{alg:fedmsa}.  In Line 6, given global update directions 
from the previous round $(\m{h}_{r-1}, \{\m{q}_{r-1}^n\}_n)$, each client $m \in [M]$ computes the local update directions $(\m{h}_r^m, \{\m{q}^{m,n}_r\}_n)$ via  the following momentum rule:
\begin{subequations}\label{eqn:moment:maps}
    \begin{align}
     \m{h}^m_r &=  \m{p}^{m} \left(\m{x}_r,\m{Z}_r;\xi^m_r\right) + (1-\rho) \left(\m{h}_{r-1} - \m{p}^{m}\left(\m{x}_{r-1},  \m{Z}_{r-1} ,\xi^m_r\right)\right), \\
    \m{q}^{m,n}_r &= \m{s}^{m,n}(\m{z}^{n-1}_{r},\m{z}^{n}_{r};\zeta^{m,n}_r)+(1-\rho)\left(\m{q}_{r-1}^n-\m{s}^{m,n}(\m{z}^{n-1}_{r-1},\m{z}^{n}_{r-1};\zeta^{m,n}_r\right),~\forall ~n \in [N].
    \end{align}
\end{subequations} 
Next, $(\m{h}_r^m, \{\m{q}^{m,n}_r\}_n)$ are communicated by all the clients to the server, which averages them to get the new global mappings $(\m{h}_r, \{\m{q}^{n}_r\}_n)$, and broadcasts them to the client $\tilde{m}$ which is chosen randomly from a uniform distribution over $M$ (Line 9). 

At client $  \tilde{m} \sim \textnormal{Unif}~ [M] $, we initialize $(\m{h}_{r,0}^{\tilde{m}}, \m{q}^{{\tilde{m}},1}_{r,0}, \ldots, \m{q}^{{\tilde{m}},N}_{r,0}) :=(\m{h}_r, \m{q}^1_r, \ldots, \m{q}^N_r)$, $\m{x}_{r+1,1}^{{\tilde{m}}}:= \m{x}_{r+1,0}^{\tilde{m}}:= \m{x}_{r}$,  and $\m{Z}_{r+1,1}^{\tilde{m}}:= \m{Z}_{r+1,0}^{\tilde{m}}:= \m{Z}_{r}$. We then compute the local mapping estimators by recalling \eqref{eqn:stoch:oracle}, for all $k \in [K]$ as follows
\begin{subequations}\label{eqn:moment:local:maps}
\begin{align}
 \m{h}^{\tilde{m}}_{r,k} &= \m{p}^{{\tilde{m}}} \left(\m{x}_{r+1,k}^{\tilde{m}}, \m{Z}_{r+1,k}^{{\tilde{m}}};\xi^m_{r,k}  \right) +  \m{h}^{\tilde{m}}_{r,k-1} - \m{p}^{{\tilde{m}}}\left(\m{x}_{r+1,k-1}^{\tilde{m}}, \m{Z}^{{\tilde{m}}}_{r+1,k-1};\xi^{\tilde{m}}_{r,k}\right), \\
\m{q}^{{\tilde{m}},n}_{r,k} &= \m{s}^{{\tilde{m}},n} \left(\m{z}^{{\tilde{m}},n-1}_{r+1,k},\m{z}^{{\tilde{m}}, n}_{r+1,k}; \zeta^{{\tilde{m}}, n}_{r,k} \right)+ \m{q}_{r,k-1}^{{\tilde{m}},n}-\m{s}^{{\tilde{m}}.n}\left(\m{z}^{{\tilde{m}},n-1}_{r+1,k-1},\m{z}^{{\tilde{m}},n}_{r+1,k-1}; \zeta^{{\tilde{m}},n}_{r,k}\right).
\end{align}
\end{subequations}
These local estimators are then used to update the local variables $\m{x}_{r+1,k}^{\tilde{m}}$ and $ \m{Z}_{r+1,k}^{{\tilde{m}}}$ (Lines 14-15). At the end of $K$ local steps, the client $\tilde{m} $ transmits its updated local models  to the server, and the server aggregate them to find the global models  $(\m{x}_{r+1}, \m{Z}_{r+1})$. 

\textbf{Variance Reduction, Momentum, and Client Selection.} The momentum-based estimators in \eqref{eqn:moment:maps} are inspired by gradient estimators from \cite{cutkosky2019momentum} and \cite{patel2022towards} for stochastic single-level non-FL and single-level FL problems, respectively. Additionally, the local update directions in \eqref{eqn:moment:local:maps} employ SARAH or SPIDER-like estimators \cite{nguyen2017sarah, fang2018spider}, originally proposed for stochastic single-level problems. The client selection process, where $\tilde{m}$ is chosen randomly from a uniform distribution over $M$, combined with variance reduction techniques, plays a crucial role in our analysis and hypergradient estimation. This combination draws inspiration from federated single-level variance reduction methods \cite{murata2021bias, karimireddy2020mime, mitra2021linear, patel2022towards}. It is important to note that while variance reduction and momentum have been studied for bilevel  and compositional problems \cite{tarzanagh2022fednest,huang2022fast,li2022local,gao2022convergence}, our proposed approaches distinguish themselves through the introduction of probable local hypergradient estimation, multiple sequence analysis ($N\geq 1$), and heterogeneous-aware theoretical analysis.

\textbf{Communication}. Algorithm~\ref{alg:fedmsa} requires two rounds of communication between the server and all clients for each iteration, i.e.,  there are two back-and-forth communications involved in each iteration. Our method uses the extra round of communication, i.e., Line~9, to update the variance/client drift reduced map $(\m{h}^m, \{\m{q}^{m,n}\}_{n})$ using the current and previous server models $(\m{x}_{r}, \m{Z}_{r})$ and $(\m{x}_{r-1}, \m{Z}_{r-1})$, respectively. 
\subsection{ Federated Bilevel Optimization}\label{sec:fedblo}
In this section we apply our generic \fedmsa to bilevel optimization. 
In {federated bilevel learning}, we consider the following  optimization problem
\begin{subequations}\label{fedblo:prob}
\begin{align}\tag{Fed-BLO}\label{fedblo_main}
\begin{array}{ll}
\underset{\m{x} \in \mb{R}^{{d}_1}}{\min} &
\begin{array}{c}
f(\m{x}):=\frac{1}{M} \sum_{m=1}^{M} f^m\left(\m{x},\m{w}^\star(\m{x})\right) 
\end{array}\\
\text{~s.t.} & \begin{array}[t]{l} \m{w}^\star(\m{\m{x}})
\in \underset{ \m{w}\in \mb{R}^{{d}_2}}{\textnormal{argmin}}~~g\left(\m{x},\m{w}\right):=\frac{1}{M}\sum_{m=1}^{M} g^m\left(\m{x},\m{w}\right). 
\end{array}
\end{array}
\end{align}
\end{subequations}
Each client in our model ($M$ total clients) can have its own individual outer and inner functions $(f^m, g^m)$ to capture objective heterogeneity. We use a stochastic oracle model, where access to local functions $(f_m, g_m)$ is obtained through stochastic sampling:
\begin{align*}
    \nonumber 
    f^m(\m{x},\m{w}) := \mb{E}_{\xi \sim \mc{A}^m}\left[f^m(\m{x}, \m{w}; \xi)\right],~~
    g^m(\m{x},\m{w}) := \mb{E}_{\zeta \sim \mc{B}^m}\left[g^m(\m{x}, \m{w}; \zeta)\right],
\end{align*} 
where $(\xi, \zeta)\sim(\mc{A}^m, \mc{B}^m)$ are stochastic samples at the $m^{\text{th}}$ client.

Under suitable assumptions, the function $g^m(\m{x},\m{w})$ is differentiable. By applying the chain rule and the implicit function theorem, we obtain the following \textit{local} gradient for any $\m{x} \in \mathbb{R}^d$ \cite[Lemma~2.1]{tarzanagh2022fednest}:
\begin{align}\label{eq:hgrad}
    \nabla f^m(\m{x})= \nabla_{\m{x}}f^m\left(\m{x},\m{w}^\star(\m{x})\right)+ \nabla^2_{\m{x}\m{w}} g(\m{x},\m{w}^\star(\m{x}))\m{v}^{m,\star}(\m{x}),
\end{align}
where  $\m{v}^{m,\star}(\m{x})\in \mb{R}^{d_2}$ is the solution to the following linear system of equations:
\begin{align}\label{eq:v_star_def}
    \nabla^2_{\m{w}}g(\m{x},\m{w}^\star(\m{x})) \m{v} = -\nabla_\m{w} f^m(\m{x},\m{w}^\star(\m{x})).
\end{align}
In light of \eqref{eq:hgrad} and \eqref{eq:v_star_def}, it becomes apparent that the computation of the local gradient of $f$ at each iteration involves two subproblems: 1) approximate solution of the inner problem $\m{w}^\star(\m{x}) \in \textnormal{argmin}_{ \m{w}\in \mb{R}^{{d}_2}} ~g\left(\m{x},\m{w}\right)$; and 2) approximate solution of the linear system in \eqref{eq:v_star_def}. This poses challenges in practical implementation of federated methods, such as gradient descent, for solving  \eqref{eq:v_star_def}. Specifically, note that the stochastic approximation of $\m{v}^{m,\star}(\m{x})$ in  \eqref{eq:v_star_def} involves the \textit{global} Hessian $\nabla^2_{\m{w}}g(\m{x},\m{w}^\star(\m{x}))$ in a nonlinear manner, which is not available at any single client. Existing federated bilevel algorithms \cite{tarzanagh2022fednest,xiao2022alternating,huang2022fast,huang2023achieving,xiao2023communication} share the limitation of using an inexact hypergradient approximation for federated optimization problems. This means that the indirect gradient component $- \nabla^2_{\m{x}\m{w}} g(\m{x},\m{w}^(\m{x}))\m{v}^{m,\star}(\m{x})$  in equation \eqref{eq:hgrad} remains fixed during local training. As a result, these methods are unable to update the local indirect gradient.  We introduce our federated framework in which the solution of the inner problem $\m{w}^\star(\m{x})$, the solutions $\{ \m{v}^{m,\star}(\m{x}) \}_m$ of the linear systems in \eqref{eq:v_star_def}, and the outer variable $\m{x}$ all evolve at the same time. This is influenced by the non-FL bilevel optimization framework introduced in~\cite{dagreou2022framework}. To do so, we define  $\m{z}:=[\m{w}~~\m{v}]$ and consider \eqref{fedmsa:prob} with the  mappings \eqref{eq:map:bilevel}:
\begin{subequations}\label{eq:map:bilevel}
\begin{align}
\hspace{-.7cm}\mb{S}^m(\m{x},\m{z})&=
\begin{bmatrix}
 \nabla_\m{w} g^m(\m{x}, \m{w})\\
    \nabla_{\m{w}}^2 g^m(\m{x}, \m{w}) \m{v} - \nabla_{\m{x}} f^m(\m{x},\m{w})
\end{bmatrix}    
,
\\
\hspace{-.7cm}\mb{P}^m(\m{x},\m{z}) &= \nabla_\m{x} f^m(\m{x},\m{w}) -\nabla_{\m{xw}}^2 g^m(\m{x},\m{w})\m{v}.    
\end{align}
\end{subequations}
Note that comparing \eqref{fedmsa:prob} and \eqref{eq:map:bilevel}, since $N=1$, we omit the index $n$ to simplify notations.  These maps are motivated by the fact that we have $ \nabla f^m(\m{x})=\mb{P}^m(\m{x}, \m{w}^\star(\m{x}), \m{v}^{m,\star}(\m{x})) $. 
This provides us with a federated BLO where we plug \cref{eq:map:bilevel} in \cref{alg:fedmsa}. 

\subsection{Federated Multi-Level Compositional Optimization}\label{section:sco}
In this section, we consider the federated multi-level compositional optimization problem 
\begin{equation}\label{eq:sc}\tag{Fed-MCO}
    \min_{\m{x}\in \mathbb{R}^{d_0}}f(\m{x}) :=  \frac{1}{M}\sum_{i=1}^M f^{N}_i(  \frac{1}{M} \sum_{i=1}^M f^{N-1}_i(\dots  \frac{1}{M} \sum_{i=1}^m f^0_i(\m{x})\dots).     
\end{equation}
where $f^{m,n} : \mathbb{R}^{d_n} \mapsto \mathbb{R}^{d_{n+1}}$ for $m \in [M]$, $n=0,1,\ldots,N$ with $d_{N+1}=1$. Only stochastic evaluations of each layer function are accessible, i.e.,
\begin{equation*}
    f^{m,n} (\m{x}) := \E_{\zeta^{m,n}}[f^{m,n} (\m{x};\zeta^{m,n})],~ m \in [M], n=0,1,\ldots,N.
\end{equation*}
where $\{ \zeta^{m,n} \}_{m,n}$ are random variables.  Here, we slightly overload the notation and use $f^{m,n} (\m{x};\zeta^{m,n})$ to represent the stochastic version of the mapping.

To solve \eqref{eq:sc}, a natural scheme is to use SGD with the gradient given by
\begin{align}\label{eq:grad sc}
\nonumber
    \nabla f(\m{x}) &= \nabla f^0 (\m{x}) \nabla f^1 (f^0(\m{x})) \dots \nabla f^N(f^{N-1}(\cdots f^0(\m{x})\cdots)),
\end{align}
where we use $$\nabla f^n(f^{n-1}(\dots f^0(\m{x})\dots)) = \nabla f^n(\m{x})|_{\m{x}=f^{n-1}(\dots f^0(\m{x})\dots)}.$$
To obtain a stochastic estimator of $\nabla f(\m{x})$, we will need to obtain the stochastic estimators for $\nabla f^n (f^{n-1}(... f^0(\m{x})...))$ for each $n$. For example, when $n=1$, one need the estimator of $\nabla f^1 (\E_{\zeta^0}[f^0(\m{x};\zeta^0)])$. However, due to the possible non-linearity of $\nabla f^1(\cdot)$, the natural candidate $\nabla f^1 (f^0(\m{x};\zeta^0))$ is not an unbiased estimator of $\nabla f^1 (\E_{\zeta^0}[f^0(\m{x};\zeta^0)])$. To tackle this issue, a popular method is to directly track  $\E_{\zeta^n}[f^n(\cdot;\zeta^n)]$ by variable $\m{z}^n, n=0,1,\ldots,N$. The mappings take the following form: 
\begin{subequations}\label{eq:fed:map:mcp}
\begin{align}
&\mb{S}^m(\m{z}^{n-1},\m{z}^n)=
\m{z}^{n}- f^{m,n-1} (\m{z}^{n-1})~~\textnormal{for all}~~ n=1, \ldots, N. \\
&\mb{P}^m(\m{x},\m{Z}) =  \nabla f^{m,0} (\m{x}) \nabla f^{m,1} (\m{z}^1) \ldots,  \nabla f^{m,N}(\m{z}^N).
\end{align}
\end{subequations}
This provides us with a second algorithm, {FedMCO}, where we plug \cref{eq:fed:map:mcp} in \cref{alg:fedmsa}.
Existing federated methods for multi-objective optimization (MCO) primarily focus on the $N=1$ \cite{huang2021compositional, tarzanagh2022fednest}. In contrast, our approach extends these methods to the multi-level federated setting, offering improved communication complexity for any $N \geq 1$. 

 \section{Convergence Analysis}\label{sec converge}
In this section, we provide the convergence guarantees of Algorithm~\ref{alg:fedmsa}. Throughout, we set  $\mb{P}(\m{x}):= \mb{P}\big(\m{x},\m{z}^{1,*}(\m{x}),\ldots,\m{z}^{N,*}(\dots \m{z}^{2,*}(\m{z}^{1,*}(\m{x}))\dots)\big)$. We make the following assumption on the fixed points and mappings.  
\begin{assumption}
\label{assum:lip:y*}
For any $m \in [M]$, $n \in \{0\} \cup [N]$ and $\m{z}^{n-1} \in \mathbb{R}^{d_{n-1}}$:
\begin{enumerate}[label={\textnormal{\textbf{A\arabic*.}}}, wide, labelwidth=!, labelindent=0pt]
\item There exists a unique $\m{z}^{n,*}(\m{z}^{n-1}) \in \mathbb{R}^{d_n}$ such that $\mb{S}^n(\m{z}^{n-1},\m{z}^{n,*}(\m{z}^{n-1}))=0$. 
\item  $\m{z}^{n,*}(\m{z}^{n-1})$ and $\nabla \m{z}^{n,*}(\m{z}^{n-1})$ are $L_{\m{z},n}$ and $L^{'}_{\m{z},n}$--Lipschitz continuous, respectively. 
\item \label{assum:lipmaps} $\mb{P}^m (\m{x})$ 
, $\mb{P}^m (\cdot, \m{Z})$, and $\mb{S}^{m,n} (\cdot, \m{z}^n)$, are  $L_p$, $L_z$,  and $L_{s,n}$ Lipschitz continuous.
\item \label{assum:lipmaps:mean} $\m{p}^m (\m{x}; \xi)$, $\m{p}^m (\cdot, \m{Z};\xi)$ 
and $\m{s}^{m,n} (\cdot, \m{z}^n; \xi)$  are  $\bar{L}_p$, $\bar{L}_z$  and $\bar{L}_{s,n}$--mean Lipschitz continuous.
\item \label{assum:stmonot:g}
$\mb{S}^{m,n}(\m{z}^{n-1},\m{z}^n)$ is one-point strongly monotone on $\m{z}^{n,*}(\m{z}^{n-1})$ given any $\m{z}^{n-1}$; that is $$\left\langle \m{z}^n\!-\!\m{z}^{n,*}(\m{z}^{n-1}),\mb{S}^{m,n}(\m{z}^{n-1},\m{z}^n) \right\rangle \leq \lambda_n \left\|\m{z}^n\!-\!\m{z}^{n,*}(\m{z}^{n-1})\right\|^2,~~\text{for some $\lambda_n>0$}.$$ 
\end{enumerate}
\end{assumption}

\begin{assumption}[Bias and variance]\label{assum:bias} 
For all $ (n,m) \in [N] \times [M]$
\begin{align*}
    \mb{E}_{\xi \sim \mc A^{m}} \left\| \m{p}^{m} \left(\m{x},\m{Z}; \xi\right) - {\mb{P}}^{m}\left(\m{x}, \m{Z}\right) \right\|^2 & \leq \sigma^2,\\
    \mb{E}_{\zeta \sim \mc B^{m,n}} \left\|
    \m{s}^{m,n}(\m{z}^{n-1},\m{z}^{n}; \zeta) - {\mb{S}}^{m,n}(\m{z}^{n-1},\m{z}^{n}) \right\|^2 & \leq \sigma_n^2, \text{ for } n \in [N]
\end{align*}
\end{assumption}
\begin{assumption}[Heterogeneity]
\label{assum:heter:h}
For all $ (n,m) \in [N] \times [M]$, the set of mappings $ \{\mb{P}^{m}\}$ and $ \{\mb{S}^{m,n}\}$ are $\tau_0$ and $\tau_n$--Heterogeneous, respectively. 
\end{assumption}
Note that Assumptions~\ref{assum:lip:y*}-- \ref{assum:heter:h} are widely used in the analysis of non-federated MSA \cite{shen2022single} federated single SA~\cite{karimireddy2020mime}.  Assumption~\ref{assum:heter:h} relates the mappings of different clients to one another and are used in the analysis of single-seqeunce SA \cite{karimireddy2020mime,murata2021bias}. Assumption~\ref{assum:stmonot:g} has also been used in previous linear \cite{kaledin2020finite} and nonlinear SA \cite{doan2020nonlinear}.

We now present the convergence guarantee of Algorithm~\ref{alg:fedmsa}.
\begin{theorem}\label{thm:fedmsa}
Suppose Assumptions~\ref{assum:lip:y*}--\ref{assum:heter:h} hold. Further, assume $\rho =\Theta(\frac{1}{R})$, $\alpha=\mc{O}(\frac{1}{\tau K})$, and $\beta_n=\mc{O}(\frac{1}{\tau K})$ for $n \in [N]$, then 
\begin{align*}
\E \left\|\mb{P}(\tilde{\m{x}})\right\|^2 + \sum_{n=1}^N\E \left\|\tilde{\m{z}}^n- \m{z}^{n,*}(\tilde{\m{z}}^{n-1})\right\|^2 \leq  \mc{O}\left(\frac{\tau}{R} +\frac{1}{\sqrt{K}R} + \frac{\sigma^2}{MKR} + \left(\frac{\sigma}{MKR}\right)^{2/3}\right). 
\end{align*}
Here, $\sigma:= \max (\sigma_1, \cdots, \sigma_N)$ and $\tau=\max(\tau_0, \cdots, \tau_N)$, and $\mc{O}$ hides problem dependent constants of a polynomial of $N$.

\end{theorem}
\textbf{Comparison with previous MSA, BLO, and MCO results}:  The convergence rate of FedMSA is guaranteed to be $\frac{\tau}{R} +\frac{1}{\sqrt{K}R} + \frac{\sigma^2}{MKR} + \left(\frac{\sigma}{MKR}\right)^{2/3}$, outperforming previous non-federated MSA~\cite{shen2022single} algorithm with the rate  $ \frac{1}{\sqrt{R}}$. 
It also achieves a notable improvement in the communication complexity of BLO and MCO, with an upper bound of $\mathcal{O}(\tau \epsilon^{-1})$, compared to existing results in bilevel \cite{tarzanagh2022fednest,xiao2022alternating,huang2022fast,huang2023achieving,xiao2023communication} and compositional \cite{huang2021compositional,tarzanagh2022fednest,huang2022faster}  optimization.   These advantages are particularly significant when $\tau$ is small, as shown in Table \ref{table:fedmsa:results}.
\section{Numerical Experiments}\label{sec:numerics}

In the experiments, we conduct experiments for the \ref{fedblo_main} and \ref{eq:sc} problems.   
\begin{wrapfigure}{r}{0.4\textwidth}
  \vspace{-.7cm}
    \begin{tikzpicture}
      \node at (0,0) {\includegraphics[scale=0.37]{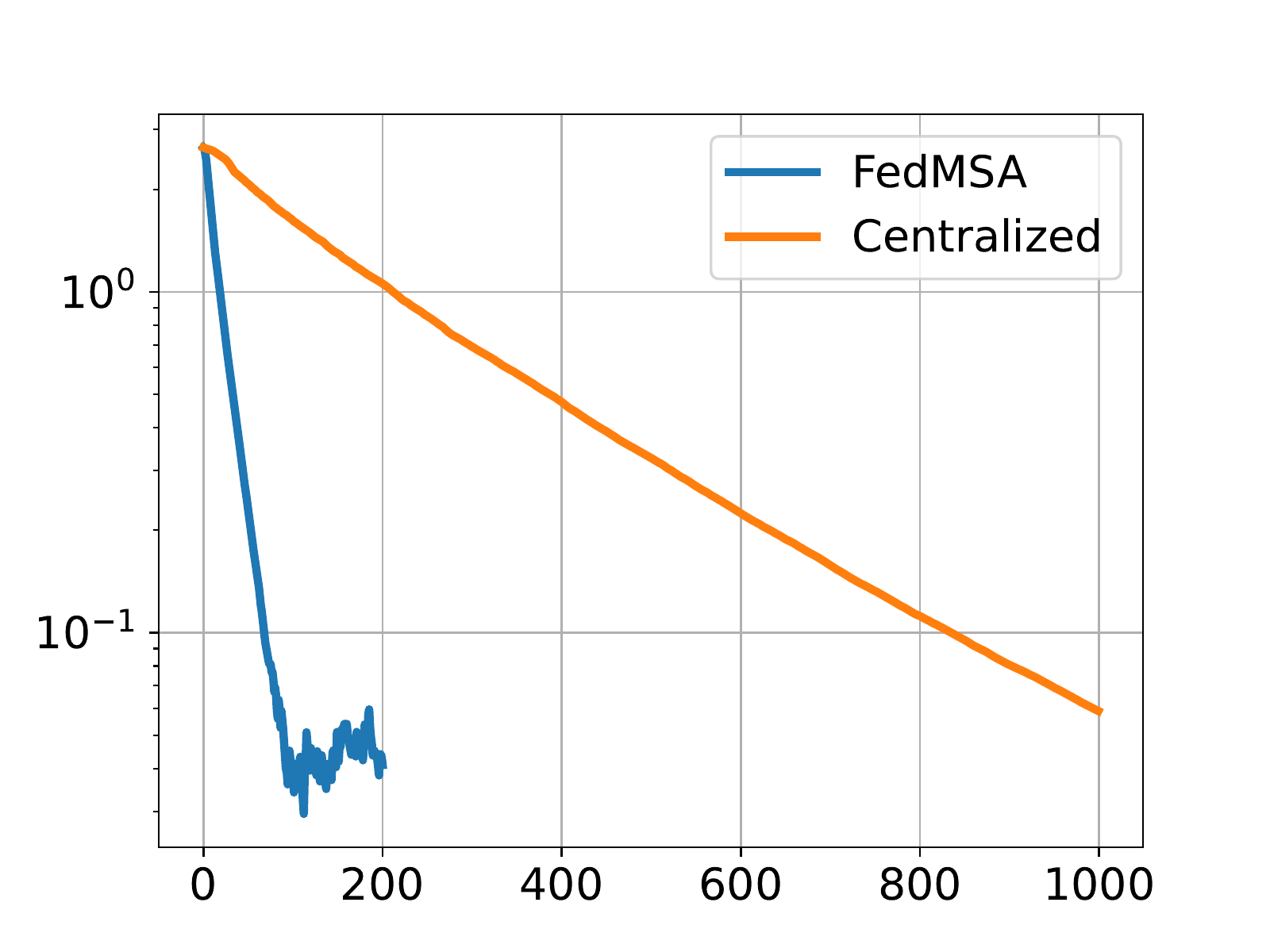}};
      \node[scale=1.0] at (0,-2.4) {Communication Rounds ($R$)};
      \node[scale=1.1, rotate=90] at (-3.2,-0.2) {$\|\m{x}-\m{x}^*\|$};
    \end{tikzpicture}
  \caption{ On \ref{eq:sc} problem, FedMSA achieves faster convergence compared to its centralized counterpart. }
  \label{fig:supp_fedmco}
  \vspace{-1cm}
\end{wrapfigure}
We first apply Algorithm~\ref{alg:fedmsa} to solve \ref{eq:sc}  problem. Our example is specifically chosen from the field of risk-averse stochastic optimization, which involves multilevel stochastic composite optimization problems. It can be formulated as follows: $$\min_{\m{x}}{\mb{E}[U(\m{x}, \xi)]+\lambda \sqrt{\mb{E}[\max(0, U(\m{x},\xi)-\mb{E} [U(\m{x},\xi)])^2]}}.$$
This problem is a stochastic three-level ($N=2$) composition optimization problem \cite{balasubramanian2022stochastic} with 
\begin{align*}
  f_0(\m{x}) &= (\m{x}, \mb{E}[U(\m{x},\xi)]),\\
  f_1(\m{x}, \m{y}) &= (\m{y}, \mb{E}[\max(0, U(\m{x}, \xi) - \m{y})]),\\
  f_2(\m{x}, \m{y}) &= \m{x}+ \lambda\sqrt{\m{y}+\delta}.  
\end{align*}
The loss function can be written in the compositional form as
$f_2(f_1(f_0(\m{x})))$.  We also note that to be consistent with \cite{balasubramanian2022stochastic}, in experiment, we also define $U(\m{x},\xi) = (b-g(\m{a}^\top \m{x}))^2$,  and $g(\m{x})=\m{x}^2$. For the experiment, we assume $\m{a}\in\mathbb{R}^d$ is a zero-mean Gaussian random vector with a random covariance matrix $\Sigma_{i,j} \sim \mathcal{N}(0,1)$, and $\zeta \sim \mathcal{N}(0, 0.001\times \m{I}_d)$ is the noise. The true parameter $\m{x}^*\in \mathbb{R}^d$ is drawn from a standard Gaussian distribution and fixed.  We set $d=10$, and the total number of examples in the experiment is set to $1,000$.

In Figure \ref{fig:supp_fedmco}, we present the results of our experiment where we conduct 1,000 iterations and compare the performance of our proposed method, FedMSA, with centralized training. In order to ensure a fair comparison in terms of the number of updates, we set the hyperparameters of FedMSA to $R=200$ and $K=5$, resulting in a total of 1,000 updates. Our analysis reveals that FedMSA exhibits a significantly faster convergence rate compared to the centralized method, demonstrating improvements in both communication efficiency and the number of updates required to reach convergence. 

\begin{figure}[t]
  \centering
  \begin{subfigure}[b]{0.33\textwidth}
    \centering
    \begin{tikzpicture}
      \node at (0,0) {\includegraphics[scale=0.3]{figs/FedBLO/tau_0.1.pdf}};
      \node[font=\small, scale=0.8] at (0,-1.8) {Local updates $(K)$};
      \node[font=\small, scale=0.8, rotate=90] at (-2.3,-0.2) {Test accuracy};
    \end{tikzpicture}
    \caption{$q=0.1$}
    \label{fig:tau_subfig_q0p1}
  \end{subfigure}%
  \begin{subfigure}[b]{0.33\textwidth}
    \centering
    \begin{tikzpicture}
      \node at (0,0) {\includegraphics[scale=0.3]{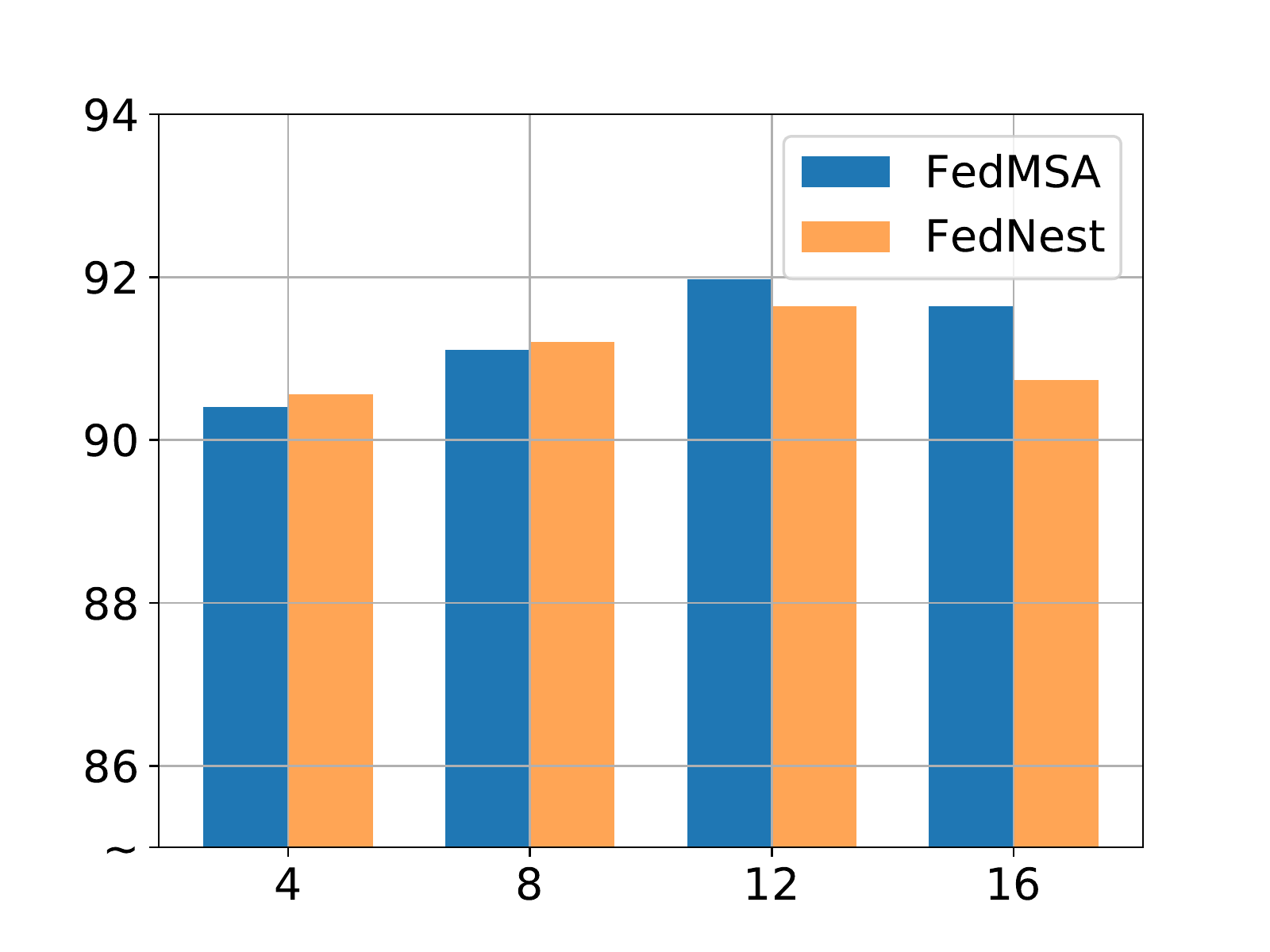}};
      \node[font=\small, scale=0.8] at (0,-1.8) {Local updates $(K)$};
    \node[font=\small, scale=0.8, rotate=90] at (-2.3,-0.2) {Test accuracy};
    \end{tikzpicture}
    \caption{$q=0.3$}
    \label{fig:tau_subfig_q0p3}
  \end{subfigure}%
  \begin{subfigure}[b]{0.33\textwidth}
    \centering
    \begin{tikzpicture}
      \node at (0,0) {\includegraphics[scale=0.3]{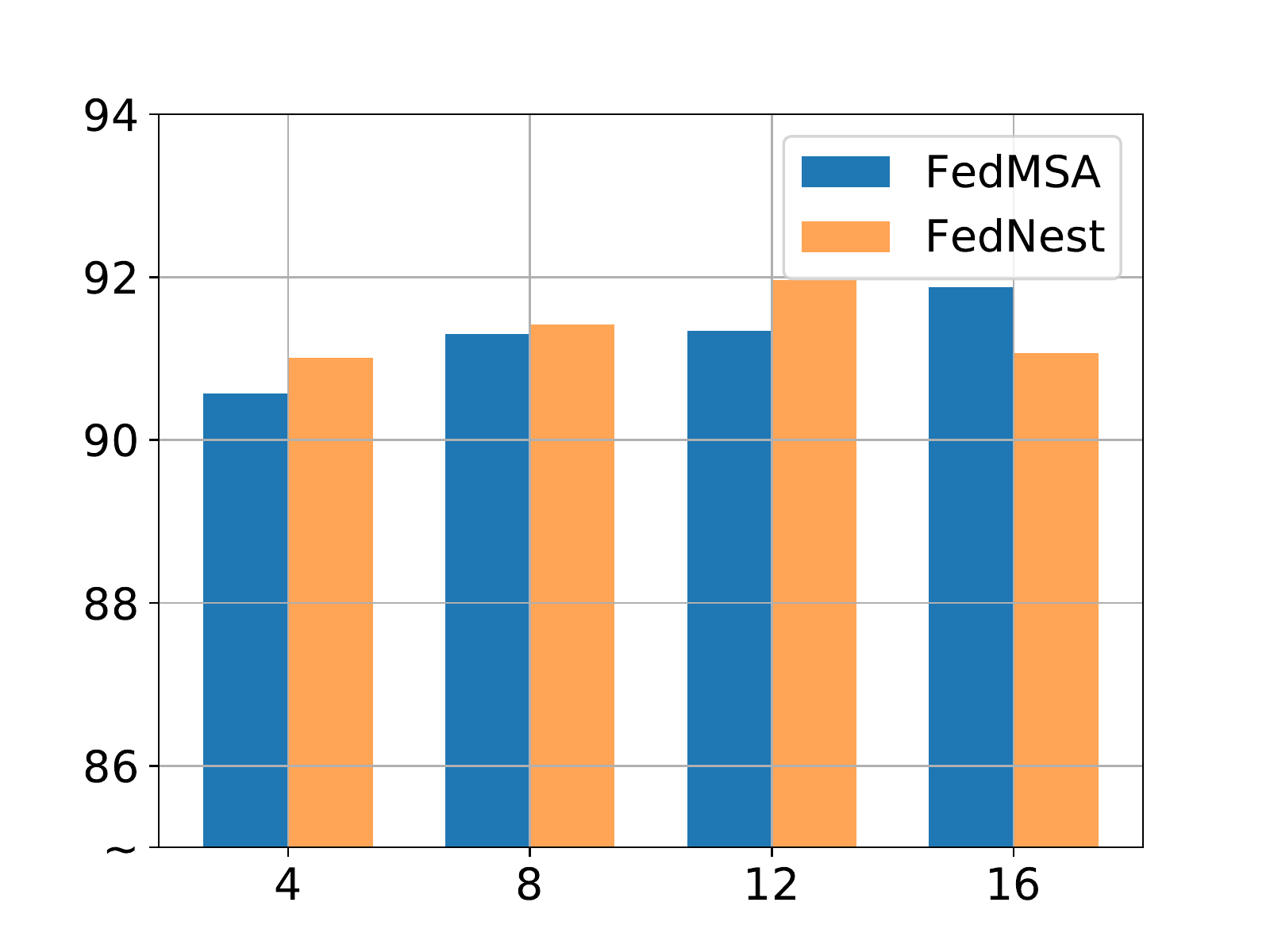}};
      \node[font=\small, scale=0.8] at (0,-1.8) {Local updates $(K)$};
     \node[font=\small, scale=0.8, rotate=90] at (-2.3,-0.2) {Test accuracy};

    \end{tikzpicture}
    \caption{$q=0.5$}
    \label{fig:tau_subfig_q0p5}
  \end{subfigure}%
  \caption{At the final of training, with the same number of server/global epochs, FedMSA benefits from local hypergradient estimation which can update the indirect component of hypergradient in local iterations.  }
  \label{fig:tau}
\end{figure}

We now focus on \ref{fedblo_main}. Our experiment follows the setup of loss function tuning on an imbalanced dataset, as described in \cite{tarzanagh2022fednest}. The objective is to maximize the class-balanced validation accuracy while training on the imbalanced dataset. We adopt the same network architecture, long-tail MNIST dataset, and train-validation strategy as \cite{tarzanagh2022fednest}. However, unlike their approach of partitioning the dataset into 100 clients using FedAvg \cite{mcmahan17fedavg} with either i.i.d. or non-i.i.d. distribution, we introduce fine-grained control over partition heterogeneity inspired by \cite{murata2021bias}. 

In a dataset with $C$ imbalanced classes, each containing $n_i$ samples, we aim to achieve a specified heterogeneity level $q \in [0,1]$. The dataset is divided into $C$ clients, each designed to have an equal number of samples, $C^{-1}\sum_{i=1}^{C} n_i$. For each client $i$, we include $q \times 100\%$ of the data from class $i$, or all $n_i$ examples if the class size is insufficient. If a client has fewer samples than $C^{-1}\sum_{i=1}^{C} n_i$, we fill the gap by uniformly sampling from the remaining samples across all classes.  This ensures an equal sample count across clients and a specified level of class heterogeneity, even within the context of an imbalanced dataset.  In our experiments, we first split the imbalanced MNIST dataset into $C=10$ clients according to $q$, and then split each client into $10$ smaller clients to match the 100 clients used in \cite{tarzanagh2022fednest, mcmahan17fedavg}. 
Throughout the experiment section, we choose 10 clients randomly from 100 clients for local updates to coincide with the literature. 

\begin{figure}[t]
  \centering
  \begin{subfigure}[b]{0.33\textwidth}
    \centering
    \begin{tikzpicture}
      \node at (0,0) {\includegraphics[scale=0.3]{figs/FedBLO/q_0.1.pdf}};
      \node[font=\small, scale=0.8] at (0,-1.8) {Communication rounds};
      \node[font=\small, scale=0.8, rotate=90] at (-2.3,-0.2) {Test accuracy};
    \end{tikzpicture}
    \caption{$q=0.1$, $K=12$}
    \label{fig:q_subfig_q0p1}
  \end{subfigure}%
  \begin{subfigure}[b]{0.33\textwidth}
    \centering
    \begin{tikzpicture}
      \node at (0,0) {\includegraphics[scale=0.3]{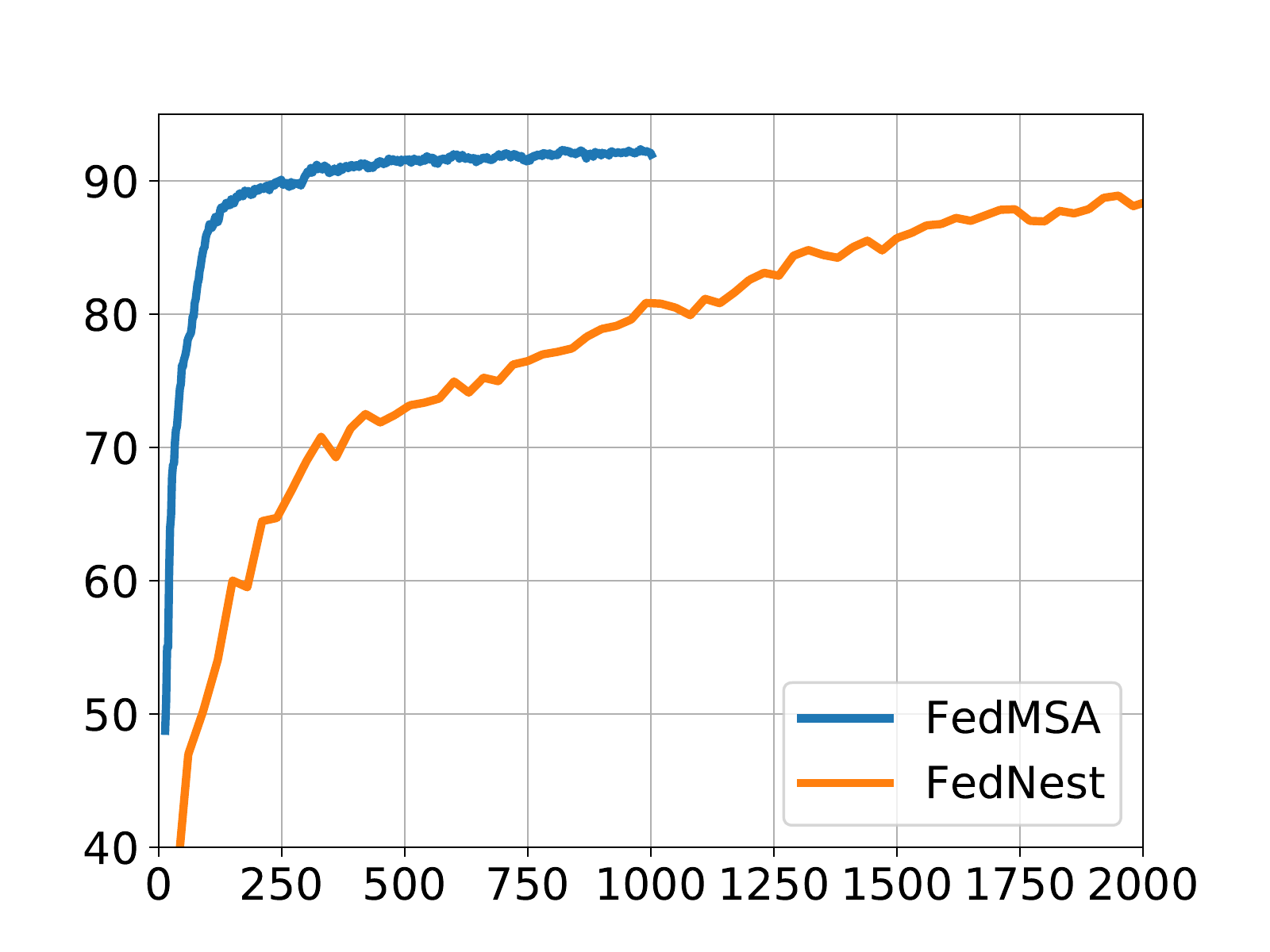}};
      \node[font=\small, scale=0.8] at (0,-1.8) {Communication rounds};
    \end{tikzpicture}
    \caption{$q=0.3$, $K=12$}
    \label{fig:q_subfig_q0p3}
  \end{subfigure}%
  \begin{subfigure}[b]{0.33\textwidth}
    \centering
    \begin{tikzpicture}
      \node at (0,0) {\includegraphics[scale=0.3]{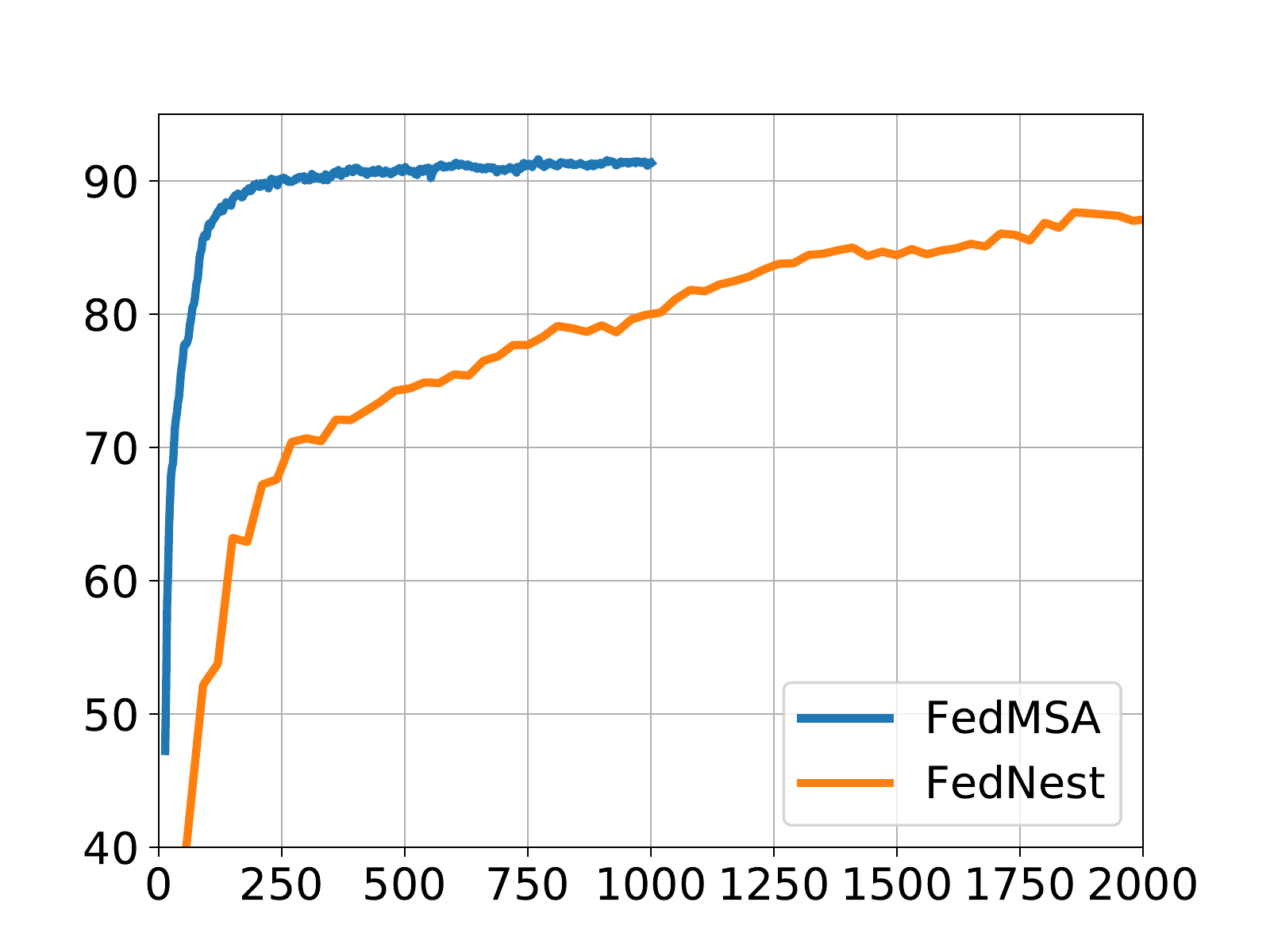}};
      \node[font=\small, scale=0.8] at (0,-1.8) {Communication rounds};
    \end{tikzpicture}
    \caption{$q=0.5$, $K=12$}
    \label{fig:q_subfig_q0p5}
  \end{subfigure}%
  \caption{The test accuracy during training, with fixed local updates ($K=12$) and different heterogeneity levels $q$, demonstrates the early stopping of FedMSA at 1000 communication rounds compared to the extended training of FedNest. }
  \label{fig:q}
\end{figure}

In Figure~\ref{fig:tau}, FedMSA's ability in local hypergradient update is demonstrated. Since the outer objective $f$ solely depends on optimal model parameters $\bm{w}^*(\bm{x})$ for this problem, the direct hypergradient $\nabla_{\m{x}}f^m\left(\m{x},\m{w}^*(\m{x})\right)$ remains zero. FedNest does not observe any change in local updates without updating the global Hessian-Gradient-Product. However, FedMSA benefits from local hypergradient estimation and updates the indirect hypergradient in local iterations, resulting in improved performance with larger local update $K$.  In Figure \ref{fig:q}, the test performance is shown for different heterogeneity levels and communication rounds. FedMSA stops at $1,000$ rounds as further training does not decrease loss. FedNest requires over $2,000$ rounds to reach the same accuracy as FedMSA's 90\% test accuracy achieved in approximately 250 rounds. The efficiency of Algo.\ref{alg:fedmsa} is highlighted by the significant reduction in communication rounds.

\section{Related Work}\label{sec related}
We gather the related work under two topics: multi-sequence stochastic approximation and federated nested optimization. A more in-depth discussion is provided in Appendix. 

\noindent\textbf{Multi-Sequence Stochastic Approximation}. 
DSA and MSA have found widespread applications in various domains, including stochastic control \cite{doan2022nonlinear}, bilevel/multi-level optimization~\cite{sato2021gradient, yang2019multilevel, balasubramanian2022stochastic,tarzanagh2022online}, minimax optimization \cite{sharma2022federated, sharma2023federated,daskalakis2018limit}, and actor-critic reinforcement learning~\cite{wu2020finite, arora2020provable}. Recent literature has proposed and analyzed several DSA and MSA methods to address these problems \cite{doan2022nonlinear, hong2020two, yang2019multilevel, sato2021gradient, shen2022single}. While many analyses of DSA focus on the linear case with linear mappings $v(\mathbf{x},\mathbf{y})$ and $h(\mathbf{x},\mathbf{y})$, notable results for Two-time-scale (TTS) linear SA have been achieved \cite{konda2004convergence, dalal2018finite, kaledin2020finite}, proving an iteration complexity of $\mathcal{O}(\epsilon^{-1})$ for $\epsilon$-accuracy. TTS nonlinear SA has also been analyzed, with \cite{mokkadem2006convergence} establishing finite-time convergence rate under asymptotic convergence of the two sequences, and \cite{doan2022nonlinear} relaxing this assumption and showing an iteration complexity of $\mathcal{O}(\epsilon^{-1.5})$, which is larger than the $\mathcal{O}(\epsilon^{-1})$ complexity of TTS linear SA. In a federated setting, these problems have not been explored due to fundamental challenges, such as accounting for problem heterogeneity and developing hyper-gradient estimators. Our work, closely related to \cite{shen2022STSA_neurips}, considers the assumption that all but the main sequence have strongly monotone increments, providing an iteration complexity of $\mathcal{O}(\epsilon^{-2})$.

\noindent\textbf{Federated Nested Optimization.} 
Classical federated learning algorithms, such as \fedavg \cite{mcmahan2017communication}, encounter convergence issues due to \textit{client drift} \cite{karimireddy2020scaffold, hsu2019measuring}, where local client models deviate from globally optimal models due to \textit{objective heterogeneity} across clients. To mitigate this drift, variance reduction methods have been proposed to maintain client estimates of the true gradient \cite{murata2021bias, karimireddy2020mime, mitra2021linear, patel2022towards}. These methods draw inspiration from variance-reduced gradient estimators, such as \cite{defazio2014saga, johnson2013accelerating, nguyen2017sarah, cutkosky2019momentum}. Recent works, including \cite{tarzanagh2022fednest, xiao2022alternating, huang2022fast, huang2023achieving, xiao2023communication}, have developed bilevel optimization approaches for \textit{homogeneous} and general \textit{heterogeneous} federated settings. Additionally, \cite{jiao2022asynchronous} investigated asynchronous distributed bilevel optimization, while \cite{yang2022decentralized, chen2022decentralized, lu2022decentralized, gao2022stochastic, terashita2022personalized,nazari2022penalty} developed bilevel programming over decentralized networks. Our FedMSA not only provides faster rates but also addresses some shortcomings of prior works on bilevel optimization when approximating local hypergradients.

\section{Conclusions and Discussion}\label{sec:conc}
In this work, we have developed FedMSA, a novel federated algorithm for stochastic approximation with multiple coupled sequences (MSA) in the context of bilevel optimization and multi-level compositional optimization. FedMSA improves upon prior theory by enabling the estimation of local mappings and hypergradients through local client updates. It achieves near-optimal communication complexity and incorporates momentum and variance reduction techniques for accelerated convergence rates. Experimental results demonstrate the empirical benefits of FedMSA, including significant reductions in communication rounds compared to previous federated BLO schemes. However, one limitation of our work is its dependence on low heterogeneity levels for fast rates and communication complexity. Furthermore, we believe that a more general result regarding local hypergradient estimation with more than one client selection can be proven, as indicated by the empirical success of our generalized FedMSA algorithm. Nevertheless, even in its current form, FedMSA with local hypergradient estimation and near-optimal communication complexity represents a state-of-the-art algorithm and paves the way for efficient decentralized algorithms for nested problems.
\bibliography{blo.bib}
\bibliographystyle{plain}
\newpage
\appendix
\addcontentsline{toc}{section}{Appendix} 

\part{}
\parttoc 


\section{ Related Work}\label{supp:sec:related}
We provide an overview of the current literature on  non-federated nested (bilevel, minmimax, and compositional) optimization and federated learning.
\subsection{Stochastic Approximation with Multiple Coupled Sequences}
Many recent analyses of the double-sequence SA focus on the linear case where $v(\m{x},\m{y})$ and $h(\m{x},\m{y})$ are linear mappings; see e.g., \cite{konda2004convergence,dalal2018finite,kaledin2020finite}. They use the so-called two-time-scale (TTS) step sizes where in one sequence is updated in the faster time scale while the other is updated in the slower time scale; that is $\lim_{k\rightarrow \infty}\alpha_k/\beta_k=0$.  Under this setting, They are able  to leverage the analysis of the single-sequence SA.  For example, \cite{kaledin2020finite} proves an iteration complexity of $\mathcal{O}(\epsilon^{-1})$ to achieve $\epsilon$-accuracy for the TTS linear SA, which is shown to be tight. With similar choice of step sizes, the TTS nonlinear SA was analyzed in \cite{mokkadem2006convergence,doan2022nonlinear}. In \cite{mokkadem2006convergence}, the finite-time convergence rate of TTS nonlinear SA was established under the assumption that the two sequences converge asymptotically. More recently, \cite{doan2022nonlinear} alleviates this assumption and shows that TTS nonlinear SA achieves an iteration complexity of $\mathcal{O}(\epsilon^{-1.5})$. However, this iteration complexity is larger than $\mathcal{O}(\epsilon^{-1})$ of the TTS linear SA. Our work is closely related to \cite{shen2022STSA_neurips} where the authors assume that  all but the main sequence have strongly
 monotone increments and  provide the $\mc{O}(\epsilon^{-2})$ iteration complexity.  However, neither problems have been explored in a federated setting so far. This is in part because federated MSA/DSA presents fundamental challenges, as we need to account for the heterogeneous nature of the problem and develop unbiased/low-variance hyper gradient estimators.
\subsection{Nested Optimization} 

\paragraph{Bilevel Optimization.} Numerous algorithms have been proposed to address bilevel nonlinear programming problems. Some methods, such as those mentioned in \cite{aiyoshi1984solution,edmunds1991algorithms,al1992global,hansen1992new,shi2005extended,lv2007penalty,moore2010bilevel}, transform the bilevel problem into a single-level optimization problem using approaches like KKT conditions or penalty function methods. Recently, alternating gradient-based approaches have gained attention for bilevel problems. These approaches estimate the hypergradient $\nabla f(\mathbf{x})$ for iterative updates, using approximate implicit differentiation (AID) and iterative differentiation (ITD) techniques. ITD-based methods, such as those proposed in \cite{maclaurin2015gradient,franceschi2017forward,finn2017model,grazzi2020iteration}, estimate $\nabla f(\mathbf{x})$ in a reverse or forward manner using techniques like automatic differentiation. AID-based methods, as in \cite{pedregosa2016hyperparameter,grazzi2020iteration,ghadimi2018approximation}, estimate the hypergradient through implicit differentiation, involving the solution of a linear system. Our algorithms follow the latter approach. Theoretical investigations of bilevel optimization encompass both asymptotic and non-asymptotic analyses. For example, \cite{franceschi2018bilevel} examines the asymptotic convergence of a backpropagation-based ITD algorithm under the assumption of strong convexity in the inner problem. \cite{shaban2019truncated} provides a similar analysis for a truncated backpropagation approach. Non-asymptotic complexity analyses have also been conducted. \cite{ghadimi2018approximation} presents a finite-time convergence analysis for an AID-based algorithm under various loss geometries. \cite{ji2021bilevel} improves the non-asymptotic analysis for AID- and ITD-based algorithms under nonconvex-strongly-convex geometry. Bounds on complexity have been explored by \cite{Ji2021LowerBA}, and stochastic bilevel algorithms have been developed with non-asymptotic analysis by \cite{ghadimi2018approximation,ji2021bilevel,hong2020two} for expected or finite-time objective functions. Further investigations include the analysis of SGD for stochastic bilevel problems \cite{chen2021closing} and online bilevel optimization \cite{tarzanagh2022online} and the study of accelerated SGD, SAGA, momentum, and adaptive-type methods for bilevel optimization \cite{chen2021single,guo2021stochastic,Khanduri2021ANA,Ji2020ProvablyFA,huang2021biadam,dagreou2022framework}.  

\paragraph{Minimax Optimization.}
Minimax optimization, a special case of bilevel optimization, has a long history dating back to \cite{brown1951iterative}. Initial works focused on the deterministic convex-concave regime \cite{nemirovski2004prox,nedic2009subgradient}, but there has been a recent surge of studies on stochastic minimax problems. The alternating version of gradient descent ascent (SGDA) has been explored, incorporating the idea of optimism \cite{daskalakis2018limit}. SGDA has been studied in the nonconvex-strongly concave setting by \cite{rafique2021weakly,lin2020gradient}. Notably, \cite{lin2020gradient} established a sample complexity of ${\cal O}(\epsilon^{-2})$ under an increasing batch size of ${\cal O}(\epsilon^{-1})$, while \cite{chen2021closing} provided the same sample complexity with a constant batch size of ${\cal O}(1)$. Accelerated SGDA algorithms have also been developed in this setting \cite{luo2020stochastic,yan2020optimal,tran2020hybrid}. Moreover, algorithms and convergence analysis have been explored for nonconvex-nonconcave minimax problems with certain benign structure, as studied in \cite{gidel2018variational,liu2019towards,yang2020global,diakonikolas2021efficient,barazandeh2021solving}.

\paragraph{Compositional Optimization.}
Stochastic compositional optimization gradient algorithms \cite{wang2017stochastic,wang2016accelerating} can be viewed as an  alternating SGD for the special compositional problem. However, to ensure convergence, the algorithms in \cite{wang2017stochastic,wang2016accelerating} use two sequences of variables being updated in two different time scales, and thus the iteration complexity of \cite{wang2017stochastic} and \cite{wang2016accelerating} is worse than ${\cal O}(\epsilon^{-2})$ of the standard SGD.  Our work is closely related to MSA \cite{shen2022single}, where an $\mc{O}(\epsilon^{-2})$ sample complexity has been established for a multi-sequence compositional optimization ($N\geq 1$) in a non-FL setting.

\subsection{Federated Nested Optimization}
FL involves training a centralized model using distributed client data, which presents challenges related to generalization, fairness, communication efficiency, and privacy \cite{mohri2019agnostic,stich2018local,yu2019parallel,wang2018cooperative,stich2019error, basu2019qsparse,nazari2019dadam,barazandeh2021decentralized}. While \fedavg~\cite{mcmahan17fedavg} addresses some of these challenges, such as high communication costs, various variants of \fedavg have been proposed to tackle other emerging issues like convergence and client drift. Approaches such as incorporating a regularization term in client objectives \cite{li2020federated}, proximal splitting \cite{pathak2020fedsplit,mitra2021linear}, variance reduction \cite{karimireddy2020scaffold,mitra2021linear}, dynamic regularization \cite{acar2021federated}, and adaptive updates \cite{reddi2020adaptive} have been explored. Several works have analyzed \fedavg in heterogeneous settings \cite{li2020federated, wang2019adaptive, khaled2019first, li2019convergence}, deriving convergence rates that depend on the level of heterogeneity. They show that the convergence rate of \fedavg deteriorates as client heterogeneity increases. The \scaffold method \cite{karimireddy2020scaffold} utilizes control variates to reduce client drift and achieves convergence rates independent of the amount of heterogeneity. Similarly, \textsc{FedNova} \cite{wang2020tackling} and \textsc{FedLin} \cite{mitra2021linear} demonstrate convergence in the presence of arbitrary local objective and system heterogeneity. Our algorithms are based on the federated single-level variance reduction methods developed in \cite{murata2021bias,karimireddy2020mime,mitra2021linear,patel2022towards} and incorporate a client selection approach.

Recent studies have explored federated minimax optimization \cite{rasouli2020fedgan, reisizadeh2020robust,deng2020distributionally,hou2021efficient, xie2021federated}. \cite{deng2021local} studies federated optimization for smooth nonconvex minimax functions, and \cite{hou2021efficient} introduces SCAFFOLD-S, a SCAFFOLD saddle point algorithm for solving strongly convex-concave minimax problems in the federated setting. However, to our knowledge, none of these works provide a rate that depends on the problem heterogeneity.  Recent works, including \cite{tarzanagh2022fednest, xiao2022alternating, huang2022fast, huang2023achieving, xiao2023communication}, have developed federated bilevel optimization approaches for \textit{homogeneous} and general \textit{heterogeneous} federated settings. Additionally, \cite{jiao2022asynchronous} investigated asynchronous distributed bilevel optimization, while \cite{yang2022decentralized, chen2022decentralized, lu2022decentralized, gao2022stochastic, terashita2022personalized} developed bilevel programming over decentralized networks. Further, \cite{huang2021compositional,tarzanagh2022fednest,huang2022faster,zhang2023federated} developed federated algorithms for nested compositional optimization. Our FedMSA not only provides faster rates but also addresses some shortcomings of prior works on bilevel and compositional optimization when approximating local hyper gradients.
 
\section{Proof for Federated Multi-Sequence Approximation}

\subsection{Auxiliary Lemmas}
We can simplify Lemmas \ref{thm:fedin} and \ref{lem:inn:diff} by removing the subscripts $\tilde{m}$ and $r$ from the definition of model parameters, as they are based on local updates. Additionally, we will use $b_0$ to denote the initial batch size, and $b$ will represent the batch size for $r \geq 2$ throughout the proofs. This notation adjustment allows for a clearer and more concise presentation of the lemmas.

The contraction property of the sequence $\{\m{z}_{k}^{n}\}_{n=1}^N$ is stated in the following lemma:
\begin{lemma}\label{thm:fedin} 
Suppose Assumptions~\ref{assum:lip:y*}  hold, and $\beta_n \leq \lambda_n/(2 L_{s,n})$ for each $n \in [N]$. Then, for each $ k \in [K]$ and $n \in [N]$,  Algorithm~\ref{alg:fedmsa} guarantees: 
\begin{align*}
    \left\|\m{z}_{k+1}^{n}- \m{z}^{n,*}(\m{z}_k^{n-1})\right\|^2& \leq \left(1-\lambda_{n} \beta_n\right) \left\|\m{z}_{k}^n-\m{z}^{n,*}(\m{z}_k^{n-1})\right\|^2\\
    &+ \beta_n \left( \lambda_{n}^{-1}+ 2\beta_n\right)\left\| \m{q}^{n}_{k}-\mb{S}^n(\m{z}_{k}^{n-1},\m{z}_{k}^n)\right\|^2.
\end{align*}
\end{lemma}

\begin{proof}
Let $\m{z}^{n,*}_k=\m{z}^{n,*}(\m{z}_k^{n-1})$. For all $n \in [N]$, we have 
\begin{align}\label{eqn-1-1:dec:inner}
    \left\|\m{z}_{k+1}^{n}-\m{z}^{n,*}_k\right\|^2 
    &= \left\|\m{z}_k^n-\m{z}_k^{n,*}\right\|^2 - 2\beta_n \left\langle \m{z}_k^n-\m{z}_k^{n,*},\m{q}_{k}^n\right\rangle + \beta_n^2 \norm{\m{q}_{k}^n}^2.
\end{align}
The second term in \eqref{eqn-1-1:dec:inner} can be bounded as
\begin{align}
    & ~~~-\left\langle \m{z}_k^n-\m{z}_k^{n,*}, \mb{S}^n(\m{z}_k^{n-1},\m{z}_k^n) + \m{q}^{n}_{k}-\mb{S}^n(\m{z}_{k}^{n-1},\m{z}_{k}^n) \right\rangle \nn \\
    &= -\left \langle \m{z}_k^n-\m{z}_k^{n,*},\mb{S}^n(\m{z}_k^{n-1},\m{z}_k^n)\right\rangle - \left\langle \m{z}_k^n-\m{z}_k^{n,*}, \m{q}^{n}_{k}-\mb{S}^n(\m{z}_{k}^{n-1},\m{z}_{k}^n)
    \right\rangle \nonumber \\
    & \leq -\lambda_{n} \|\m{z}_k^n-\m{z}_k^{n,*}\|^2 + \left\|\m{z}_k^n-\m{z}_k^{n,*}\right\|\left\| \m{q}^{n}_{k}-\mb{S}^n(\m{z}_{k}^{n-1},\m{z}_{k}^n)\right\| \nonumber \\
    &\leq -\lambda_{n} \left\|\m{z}_k^n-\m{z}_k^{n,*}\right\|^2 + \frac{\lambda_{n}}{2}\left\|\m{z}_k^n-\m{z}_k^{n,*}\right\|^2 + \frac{1}{2\lambda_{n}} \left\|\m{q}^{n}_{k}-\mb{S}^n(\m{z}_{k}^{n-1},\m{z}_{k}^n)\right\|^2 \nonumber \\
    &\leq - \frac{\lambda_{n}}{2} \left\|\m{z}_k^n-\m{z}_k^{n,*}\right\|^2 + \frac{1}{2\lambda_{n}} \left\| \m{q}^{n}_{k}-\mb{S}^n(\m{z}_{k}^{n-1},\m{z}_{k}^n)\right\|^2. \label{eqn-1-2:dec:inner}
\end{align}
where the first inequality uses Using Assumption~\ref{assum:stmonot:g} and Cauchy-Schwarz, and the second inequality follows from Young's inequality with $\lambda_n$.

For the third term in \eqref{eqn-1-1:dec:inner}, we have 
\begin{align}
\nonumber 
    \beta_n^2 \norm{\m{q}_{k}^n}^2 & \leq 2 \beta_n^2 \left( \norm{\mb{S}^n(\m{z}_k^{n-1},\m{z}_k^n) - \mb{S}^n \left( \m{z}_k^{n-1}, \m{z}_k^{n,*} (\m{z}_k^{n-1}) \right)}^2 + \left\|\m{q}^{n}_k-\mb{S}^n(\m{z}_k^{n-1},\m{z}_k^n)\right\|^2\right) \\
    \nonumber 
    &\leq 2 L_{s,n}^2 \beta_n^2 \norm{\m{z}_k^n - \m{z}_k^{n,*} (\m{z}_k^{n-1})}^2 + 2  \beta_n^2  \left\| \m{q}^{n}_{k}-\mb{S}^n(\m{z}_{k}^{n-1},\m{z}_{k}^n)\right\|^2  \\
    &\leq \norm{\m{z}_k^n - \m{z}_k^{n,*} (\m{z}_k^{n-1})} + 2  \beta_n^2  \left\| \m{q}^{n}_{k}-\mb{S}^n(\m{z}_{k}^{n-1},\m{z}_{k}^n)\right\|^2. \label{eqn-1-3:dec:inner}
\end{align}
Here, the first inequality uses Young's inequality and the fact that and $\mb{S}^n \left( \m{z}_k^{n-1}, \m{z}_k^{n,*} (\m{z}_k^{n-1}) \right) = 0$; the second inequality follows from  Assumption~\ref{assum:lipmaps}.

Substituting  \eqref{eqn-1-2:dec:inner}  and \eqref{eqn-1-3:dec:inner} into \eqref{eqn-1-1:dec:inner} gives the desired result. 
\end{proof}

\begin{lemma}\label{lem:inn:diff} 
Suppose Assumption~\ref{assum:lip:y*} holds. Then, for each $ k \in [K]$, $\alpha \leq 1/(2L_{\m{z},1}^2)$,  Algorithm~\ref{alg:fedmsa} guarantees:
\begin{align}\label{eqn:thm:deczn00}
\nonumber
\sum_{n=1}^N \left\|\m{z}_{k+1}^{n,*}-\m{z}_{k+1}^n\right\|^2 &-\left\|\m{z}_{k}^n-\m{z}_{k}^{n,*}\right\|^2 \leq \sum_{n=1}^N C^n \left\|\m{z}_{k}^n-\m{z}_{k}^{n,*}\right\|^2 \\
&+  \sum_{n=1}^N D^n\left\| \m{q}_k^{n}- \mb{S}^n(\m{z}_k^{n-1},\m{z}_k^{n})\right\|^2\\
\nonumber
&+ 2L_{\m{z},1}^2 \alpha^2 \left( \left\|\m{h}_k-\mb{P} (\m{x}_{k})\right\|^2 + \|\mb{P}(\m{x}_{k})\|^2\right)+\frac{\alpha}{8} \|\m{h}_k\|^2.
\end{align}
where $ \{C^n,D^n\}_{n=1}^N$ are defined in \eqref{eqn:cndn}.
\end{lemma}

 \begin{proof}
Note that 
\begin{align}\label{eqn-1z1:dec:inner}
    \|\m{z}^{n}_{k+1}-\m{z}^{n,*}_{k+1}\|^2 =& \|\m{z}^{n}_{k+1}-\m{z}^{n,*}_{k}\|^2  + \|\m{z}^{n,*}_{k}-\m{z}^{n,*}_{k+1}\|^2+2 \left\langle \m{z}^{n}_{k+1} - \m{z}^{n,*}_{k}, \m{z}^{n,*}_{k} - \m{z}^{n,*}_{k+1} \right\rangle.
\end{align} 
\textbf{Case $\g{n=1}$}: Let   $\hat{\m{x}}_{k+1}=a\m{x}_{k} + (1-a)\m{x}_{k+1}, a \in [0,1]$. We have 
\begin{align}
\nonumber 
\left\langle \m{z}_{k+1}^1-{\m{z}^{1,*}_{k}}, {\m{z}^{1,*}_{k}}-{\m{z}^{1,*}_{k+1}}\right\rangle &= -\left\langle \m{z}_{k+1}^1-{\m{z}^{1,*}_{k}}, \nabla \m{z}^{1,\star}(\hat{\m{x}}_{k+1})(\m{x}_{k+1}-\m{x}_{k})\right\rangle\\
\label{eqn-1z2:dec:inner}
 &\leq  4 L_{\m{z},1}^2 \alpha \left\|\m{z}_{k+1}^1-\m{z}^{1,\star}_k\right\| ^2+\! \frac{1}{16}\alpha\|\m{h}_k\|^2 
\end{align}
Note that  $\m{z}_k^{1,*}=\m{z}^{1,*}(\m{x}_k)$. The last term in \eqref{eqn-1z1:dec:inner} can be bounded as
\begin{equation}\label{eqn-1z4:bound:dec:inner}
\begin{aligned}
\|\m{z}_k^{1,*}-\m{z}_{k+1}^{1,*}\|^2 &\leq L_{\m{z},1}^2 \|\m{x}_k-\m{x}_{k+1}\|^2\\
&\leq L_{\m{z},1}^2 \alpha^2 \left\|\mb{P} (\m{x}_k)+\m{h}_k-\mb{P}(\m{x}_{k})\right\|^2 \\
    &\leq  2 L_{\m{z},1}^2 \alpha^2 \| \mb{P} (\m{x}_k)\|^2 + 2 L_{\m{z},1}^2 \alpha^2 \|\m{h}_k-\mb{P}(\m{x}_{k})\|^2,
\end{aligned}
\end{equation}
where the second inequality uses Assumption~\ref{assum:lip:y*}.
\\
\textbf{Case $\g{n\geq 2}$}: By the mean-value theorem, for some $\hat{\m{z}}_{k+1}^{n-1}=a \m{z}_k^{n-1} + (1-a) \m{z}_{k+1}^{n-1}, a \in [0,1]$, the last term in \eqref{eqn-1z1:dec:inner} can be rewritten as
\begin{subequations}
\begin{align}
\label{eqn:inner:drift:0}
    \langle \m{z}_k^{n,*}-\m{z}_{k+1}^n,\m{z}_{k+1}^{n,*}-\m{z}_k^{n,*} \rangle 
    &= \langle \m{z}_k^{n,*}-\m{z}_{k+1}^n,\nabla \m{z}^{n,*}(\hat{\m{z}}_{k+1}^{n-1})^\top (\m{z}_{k+1}^{n-1}-\m{z}_k^{n-1}) \rangle \\
    &= \left\langle \m{z}_k^{n,*}-\m{z}_{k+1}^n,\beta_{n-1}\nabla \m{z}^{n,*}(\hat{\m{z}}_{k+1}^{n-1})^\top \mb{S}^{n-1}(\m{z}_k^{n-2},\m{z}_k^{n-1}) \right \rangle \label{eqn:inner:drift:1}\\
    &+ \left\langle \m{z}_k^{n,*}-\m{z}_{k+1}^n,\beta_{n-1}\nabla \m{z}^{n,*}(\hat{\m{z}}_{k+1}^{n-1})^\top \left( \m{q}_k^{n-1}- \mb{S}^{n-1}(\m{z}_k^{n-2},\m{z}_k^{n-1})\right)\right\rangle.
    \label{eqn:inner:drift:2}
\end{align}
\end{subequations}

Now, from Assumption~\ref{assum:lip:y*} and Assumption~\ref{assum:lipmaps}, we get
\begin{align}\label{eq:driftI1 n>=2}
     \eqref{eqn:inner:drift:1}
    &\leq L_{\m{z},n}\beta_{n-1}\| \m{z}_k^{n,*}-\m{z}_{k+1}^n\| \| \mb{S}^{n-1}(\m{z}_k^{n-2},\m{z}_k^{n-1})\| \nonumber\\
    &\leq L_{\m{z},n}L_{\m{q},n-1}\beta_{n-1}\| \m{z}_k^{n,*}-\m{z}_{k+1}^n\|\|\m{z}_k^{n-1}-\m{z}_k^{n-1,*}\| \nonumber\\
    &\leq \frac{2 L_{\m{z},n}^2 L_{\m{q},n-1}^2}{\lambda_{n-1}}\beta_{n-1}\| \m{z}_k^{n,*}-\m{z}_{k+1}^n\|^2 + \frac{\lambda_{n-1}}{8} \beta_{n-1}\|\m{z}_k^{n-1}-\m{z}_k^{n-1,*}\|^2
\\
\eqref{eqn:inner:drift:2}
    &\leq L_{\m{z},n}\beta_{n-1}\| \m{z}_k^{n,*}-\m{z}_{k+1}^n\|   \left\| \m{q}_k^n- \mb{S}^{n-1}(\m{z}_k^{n-2},\m{z}_k^{n-1})\right\|^2 \nonumber\\
    &\leq \frac{2 L_{\m{z},n}^2 L_{\m{q},n-1}^2}{\lambda_{n-1}}\beta_{n-1}\| \m{z}_k^{n,*}-\m{z}_{k+1}^n\|^2 + \frac{\lambda_{n-1}}{8} \beta_{n-1} \left\| \m{q}_k^{n-1}- \mb{S}^{n-1}(\m{z}_k^{n-2},\m{z}_k^{n-1})\right\|^2.
\end{align}
Hence, \eqref{eqn:inner:drift:0} can be bounded as  
\begin{align}\label{eqn-1z4:bound:dec:inner1}
\left\|\m{z}_{k+1}^{n,*}-\m{z}_{k+1}^n\right\|^2 &\leq \left(\frac{ 8 L_{\m{z},n}^2 L_{\m{q},n-1}^2}{\lambda_{n-1}}\beta_{n-1}\right)\| \m{z}_k^{n,*}-\m{z}_{k+1}^n\|^2 \\
 \nonumber
      &+ \frac{\lambda_{n-1}}{4} \beta_{n-1}\|\m{z}_k^{n-1}-\m{z}_k^{n-1,*}\|^2 + \frac{\lambda_{n-1}}{4} \beta_{n-1} \left\| \m{q}_k^{n-1}- \mb{S}^{n-1}(\m{z}_k^{n-2},\m{z}_k^{n-1})\right\|^2. 
\end{align}
Further, the second term in \eqref{eqn-1z1:dec:inner} can be bounded as
\begin{equation}\label{eqn-1z4:bound:dec:inner2}
\begin{aligned}
\|\m{z}_k^{n,*}-\m{z}_{k+1}^{n,*}\|^2 &\leq L_{\m{z},n}^2 \|\m{z}_k^{n-1}-\m{z}_{k+1}^{n-1}\|^2\\
    &\leq  2  L_{\m{z},n-1}^2 \beta^2_{n-1}  \| \mb{S}^{n-1}(\m{z}_k^{n-2},\m{z}_k^{n-1})\|^2\\
    &+ 2  L_{\m{z},n-1}^2 \beta^2_{n-1} \|\m{q}_k^{n-1}- \mb{S}^{n-1}(\m{z}_k^{n-2},\m{z}_k^{n-1})\|^2,\\
        &\leq  2  L_{\m{z},n-1}^2  L_{\m{q},n-1}^2 \beta^2_{n-1}  \|  \m{z}_k^{n-1,*}-\m{z}_{k+1}^{n-1}\|^2\\
    &+ 2  L_{\m{z},n-1}^2 \beta^2_{n-1} \|\m{q}_k^{n-1}- \mb{S}^{n-1}(\m{z}_k^{n-2},\m{z}_k^{n-1})\|^2,
\end{aligned}
\end{equation}
where the second inequality uses Assumption~\ref{assum:lip:y*}, and the last inequality uses Assumption~\ref{assum:lipmaps} and the fact that $\mb{S}^{n-1}(\m{z}_k^{n-2},\m{z}_k^{n-1,\star})=0$.

Substituting \eqref{eqn-1z4:bound:dec:inner2}  and \eqref{eqn-1z4:bound:dec:inner1} into \eqref{eqn-1z1:dec:inner} gives 
\begin{align}\label{eqn:err:fedin2}
\nonumber
\left\|\m{z}_{k+1}^{1}\!\!-\m{z}^{1,*}_{k+1}\right\|^2 &\leq 8L_{\m{z},1}^2 \alpha \left\|\m{z}_{k+1}^1-\m{z}^{1,\star}_k\right\| ^2 +\frac{\alpha}{8} \|\m{h}_k\|^2\\
&+2L_{\m{z},1}^2 \alpha^2 \|\nabla f (\m{x}_{k})\|^2+ 2L_{\m{z},1}^2 \alpha^2 \left\|\m{h}_k-\mb{P} (\m{x}_{k})\right\|^2,    
\\
\nonumber
\left\|\m{z}_{k+1}^{n,*}-\m{z}_{k+1}^n\right\|^2 &\leq 
\left( 2  L_{\m{z},n-1}^2  L_{\m{q},n-1}^2 \beta_{n-1}^2 +\frac{\lambda_{n-1}}{4} \beta_{n-1}\right)\|\m{z}_k^{n-1}-\m{z}_k^{n-1,*}\|^2 \\
\nonumber
&+\left(1+ { 8L_{\m{z},n}^2 L_{\m{q},n-1}^2} \lambda_{n-1}^{-1} \beta_{n-1}\right)  \| \m{z}_k^{n,*}-\m{z}_{k+1}^n\|^2
\\
&+ \left( 2  L_{\m{z},n-1}^2 \beta_{n-1} + \frac{\lambda_{n-1}}{4}\right)\beta_{n-1}  \left\| \m{q}_k^{n-1}- \mb{S}^{n-1}(\m{z}_k^{n-2},\m{z}_k^{n-1})\right\|^2.
\end{align}
Subtracting the above inequalities from each other and simplifying, for each $N\geq n \geq 2$, we obtain:
\begin{subequations}
\begin{equation}\label{eqn:thm:deczn}
\begin{aligned}
& \qquad \left\|\m{z}_{k+1}^{n,*}-\m{z}_{k+1}^n\right\|^2-\left\|\m{z}_{k}^{n,*}-\m{z}_{k}^n\right\|^2 
\\
&\leq \left( 2  L_{\m{z},n-1}^2  L_{\m{q},n-1}^2 \beta_{n-1}^2 +\frac{\lambda_{n-1}}{4} \beta_{n-1}\right)\|\m{z}_k^{n-1}-\m{z}_k^{n-1,*}\|^2\\
&+\left(\left(1+ { 8L_{\m{z},n}^2 L_{\m{q},n-1}^2} \lambda_{n-1}^{-1} \beta_{n-1}\right)\left(1-\lambda_{n} \beta_n\right)-1\right) \left\|\m{z}_{k}^n-\m{z}_{k}^{n,*}\right\|^2\\
&+ \left(1+ { 8L_{\m{z},n}^2 L_{\m{q},n-1}^2} \lambda_{n-1}^{-1} \beta_{n-1}\right)
\left( \lambda_{n}^{-1}+ 2\beta_n\right) \beta_n \left\| \m{q}^{n}_{k}-\mb{S}^n(\m{z}_{k}^{n-1},\m{z}_{k}^n)\right\|^2
\\
&+ \left( 2  L_{\m{z},n-1}^2 \beta_{n-1} + \frac{\lambda_{n-1}}{4}\right)\beta_{n-1}  \left\| \m{q}_k^{n-1}- \mb{S}^{n-1}(\m{z}_k^{n-2},\m{z}_k^{n-1})\right\|^2.
\end{aligned}
\end{equation}
Similarly, for the case $n=1$, we get:
\begin{equation}\label{eqn:thm:decz1}
\begin{aligned}
& \qquad \left\|\m{z}_{k+1}^{1}\!\!-\m{z}^{1,*}_{k+1}\right\|^2 - \left\|\m{z}_{k}^1-\m{z}_{k}^{1,*}\right\|^2\\
&\leq \left(8L_{\m{z},1} \alpha\left(1-\lambda_{1} \beta_1\right) -1\right)\left\|\m{z}_{k}^1-\m{z}_{k}^{1,*}\right\|^2\\
&+
8L_{\m{z},1} \alpha\beta_1 \left( \lambda_{1}^{-1}+ 2\beta_1\right)\left\| \m{q}^{1}_{k}-\mb{S}^1(\m{z}_{k}^{0},\m{z}_{k}^1)\right\|^2 \\
&+\frac{\alpha}{8} \|\m{h}_k\|^2+2L_{\m{z},1}^2 \alpha^2 \|\mb{P} (\m{x}_{k})\|^2+ 2L_{\m{z},1}^2 \alpha^2 \left\|\m{h}_k-\mb{P} (\m{x}_{k})\right\|^2.
\end{aligned}
\end{equation}
\end{subequations}
To simplify our next proofs, we define:
\begin{equation}\label{eqn:cndn}
    \begin{split}
         C^1(\alpha, \beta)&:=  \left(8L_{\m{z},1} \alpha\left(1-\lambda_{1} \beta_1\right) -1\right)+ \left( 2  L_{\m{z},1}^2  L_{\m{q},1}^2 \beta_{1}^2 +\frac{\lambda_{1}}{4} \beta_{1}\right) 
 \\
C^n(\alpha, \beta)&:= \left(\left(1+ { 8L_{\m{z},n}^2 L_{\m{q},n-1}^2} \lambda_{n-1}^{-1} \beta_{n-1}\right)\left(1-\lambda_{n} \beta_n\right)-1\right) \\
&+ \left( 2  L_{\m{z},n-1}^2  L_{\m{q},n-1}^2 \beta_{n-1}^2 +\frac{\lambda_{n-1}}{4} \beta_{n-1}\right) , ~~\textnormal{for}~~n=2, \ldots, N-1,\\
C^N(\alpha, \beta)&:= \left(\left(1+ { 8L_{\m{z},N}^2 L_{\m{q},N-1}^2} \lambda_{N-1}^{-1} \beta_{N-1}\right)\left(1-\lambda_N \beta_N\right)-1\right).\\
 D^1(\alpha, \beta)&:= 8L_{\m{z},1} \alpha\beta_1 \left( \lambda_{1}^{-1}+ 2\beta_1\right) +  \left( 2  L_{\m{z},1}^2 \beta_{1} + \frac{\lambda_{1}}{4}\right)\beta_{1}, ~~\textnormal{for}~~n=2, \ldots, N,\\
 D^n(\alpha, \beta)&:= \left(1+ { 8L_{\m{z},n}^2 L_{\m{q},n-1}^2} \lambda_{n-1}^{-1} \beta_{n-1}\right)
\left( \lambda_{n}^{-1}+ 2\beta_n\right) \beta_n \\
&+ \left( 2  L_{\m{z},n-1}^2 \beta_{n-1} + \frac{\lambda_{n-1}}{4}\right)\beta_{n-1} ,~~\textnormal{for}~~~n=2, \ldots, N-1,\\
 D^N(\alpha, \beta)&:= \left(1+ { 8L_{\m{z},N}^2 L_{\m{q},N-1}^2} \lambda_{N-1}^{-1} \beta_{N-1}\right)
\left( \lambda_{N}^{-1}+ 2\beta_N\right) \beta_N.
    \end{split}
\end{equation}
Combining \eqref{eqn:thm:deczn} with \eqref{eqn:thm:decz1} and using our choice of $\alpha$ gives \eqref{eqn:thm:deczn00}. 
\end{proof}

\begin{lemma}\label{lem:drift:fedmsa}
For each $m \in [M]$ and $ k \in [K]$,  Algorithm~\ref{alg:fedmsa} guarantees:
\begin{subequations}
\begin{equation}
\begin{aligned}\label{eqn:lem:drift:fedmsav1}
\left\|\m{x}^{m}_{r+1,k}-\m{x}_r\right\|^2 & \leq   2 e \alpha^2 (1+K) \sum\limits_{k=1}^{K-1} \left\| \mb{P}(\m{x}_{r+1,k}^{m})\right\|^2\\
&+2 e \alpha^2 (1+K) \sum\limits_{k=1}^{K-1}  \left\|\mb{P}(\m{x}_{r+1,k}^{m})-\m{h}_{r,k}^{m} \right\|^2,
\end{aligned} 
\end{equation}
\begin{equation}
\begin{aligned}\label{eqn:lem:drift:fedmsav2}
\left\|\m{z}^{m,n}_{r+1,k}-\m{z}_r^n\right\|^2 & \leq   2 e \beta^2_n (1+K) \sum\limits_{k=1}^{K-1} \left\| \mb{S}(\m{x}_{r+1,k}^{m})\right\|^2\\
&+2 e \beta^2_n (1+K) \sum\limits_{k=1}^{K-1}  \left\| \mb{S}^n(\m{z}_{r+1,k}^{m,n-1},\m{z}_{r+1,k}^{m,n})-\m{q}^{m,n}_{r+1,k}\right\|^2,
\end{aligned} 
\end{equation}
\end{subequations}
where $e:=\exp(1)$.
\end{lemma}
\begin{proof}
From the local updates in  Algorithm~\ref{alg:fedmsa}, we have that $\forall m\in [M]$, $k \in [K]$
\begin{equation}
\label{eqn2:lem:drift:fedmsav2}
\begin{aligned}
 \left\|\m{x}^{m}_{r+1,k}-\m{x}_r\right\|^2 &\leq  \left(1+\frac{1}{K}\right)\|\m{x}_{r+1,k-1}^m-\m{x}_r\|^2+ (K+1)\alpha^2 \left\|\m{h}_{r,k-1}^m\right\|^2\\
        & \leq  \left(1+\frac{1}{K}\right)\left\|\m{x}_{r+1,k-1}^m-\m{x}_r\right\|^2 +  2\alpha^2 (1+K)  \left\|\mb{P}(\m{x}_{r+1,k-1}^m ) \right\|^2
        \\
       &+ 2\alpha^2 (1+K)  \left\| \m{h}^m_{r,k-1}-\mb{P}(\m{x}_{r+1,k-1}^m ) \right\|^2.
    \end{aligned}
\end{equation}
%
%

Now, iterating equation \eqref{eqn2:lem:drift:fedmsav2} and using $  \m{x}^{m}_{r+1,1}=\m{x}^{m}_{r+1,0}=\m{x}_r ~~\forall m\in [M]$, we obtain 
\begin{equation}
\label{eqn3:lem:drift:fedmsa}
\begin{aligned}
\eqref{eqn2:lem:drift:fedmsav2}
       &\leq  2\alpha^2 (1+K) \sum\limits_{j=2}^{k} \left(1+\frac{1}{K+1}\right)^j \left[ \left\| \m{h}^m_{r,j-1}- \mb{P} (\m{x}_{r+1,j-1}^m) \right\|^2 + \left\| \mb{P}(\m{x}_{r+1,j-1}^m)  
       \right\|^2 \right]\\
       & \leq 2 e \alpha^2 (1+K) \sum\limits_{k=2}^{K} 
       \left[ \left\| \m{h}^m_{r,k-1}- \mb{P} (\m{x}_{r+1,k-1}^m) \right\|^2 + \left\| \mb{P}(\m{x}_{r+1,k-1}^m)  
       \right\|^2 \right]\\
        & \leq 2 e \alpha^2 (1+K) \sum\limits_{k=1}^{K-1} 
        \left[ \left\| \m{h}^m_{r,k}- \mb{P} (\m{x}_{r+1,k}^m) \right\|^2 + \left\| \mb{P}(\m{x}_{r+1,k}^m)  
       \right\|^2 \right].
\end{aligned}
\end{equation}
\end{proof}
The following lemma bounds the deviation of $\m{h}^m_{r,k}$
from $\mb{P} (\m{x}_{r+1,k}^m )$.

\begin{lemma}\label{lem:drift:fedmsa2}
Suppose Assumptions~ \ref{assum:lip:y*}--\ref{assum:heter:h} hold. Further, assume $K \geq 1$ and $\alpha \leq \min ( CK \nu_0, C'\sqrt{K}\bar{L}_p)$ and $\beta_n \leq \min (DK \nu_n, D'\sqrt{K}\bar{L}_{s,n})$ for all $n \in [N]$ and some constants $C,C',D,D'$. Then, for each $m \in [M]$ and $ k \in [K]$,  Algorithm~\ref{alg:fedmsa} guarantees:
\begin{equation}\label{eqn7:lem:drift:fedmsa}
\begin{aligned}
\frac{1}{K}\sum_{k=1}^K\mb{E}\left\| \m{q}_{r,k}^{m,n} - \mb{S}^{n}(\m{z}_{r+1,k}^{m,n-1},\m{z}_{r+1,k}^{m,n})\right\|^2&\leq  4 e \mb{E}\left\|\m{q}_{r}^{n} - \mb{S}^{n}(\m{z}_{r}^{n-1},\m{z}_{r}^{n})\right\|^2\\
&  +\frac{1}{K} \sum_{k=1}^K \mb{E}  \left\|\m{z}_{r+1,k}^{m,n} - \m{z}_{r+1,k}^{m,n,*} \right\|^2,
\\
\frac{1}{K}\sum_{k=1}^K\mb{E}\left\|\m{h}^m_{r,k}- \mb{P} (\m{x}_{r+1,k}^m)\right\|^2 &\leq  4 e \mb{E}\left\|\m{h}_{r}- \mb{P} (\m{x}_{r})\right\|^2+\frac{1}{K} \sum_{k=1}^K \mb{E} \left\| \mb{P}(\m{x}_{r+1,k}^m)\right\|^2\\
&+ 2\sum_{k=1}^K L_{\m{h},\m{z}} \sum_{n=1}^N L_\m{z} (n)\|\m{z}_{r+1,k}^{m,n} - \m{z}_{r+1,k}^{m,n,*}\|.
\end{aligned}
\end{equation}
\end{lemma}

\begin{proof}

Note that $\mb{P} (\m{x})=\mb{P} (\m{x}, \m{Z}^\star)$  and $\m{p} (\m{x})=\m{p} (\m{x}, \m{Z}^\star)$.
Let
\begin{equation}\label{eqn5:lem:drift:fedmsa}
\begin{aligned}
\m{a}_{r,k}^m &= \m{p}^m(\m{x}_{r+1,k}^m, \m{Z}_{r+1,k}^m; \xi^m_r) - \m{p}^m(\m{x}_{r+1,k}^m;\xi_{r+1,k}^m)\\
&-\m{p}^m  (\m{x}_{r+1,k-1}^m, \m{Z}_{r+1,k-1}^m;\xi^m_{r,k} ) + \m{p}^m(\m{x}_{r+1,k-1}^m,\xi_{r+1,k}^m),
\\
 \m{b}_{r,k}^m &= \m{p}^m(\m{x}_{r+1,k}^m,\xi_{r+1,k}^m) -\m{p}^m(\m{x}_{r+1,k-1}^m,\xi_{r+1,k}^m)- \mb{P}^m(\m{x}_{r+1,k}^m) +\mb{P}^m (\m{x}_{r+1,k-1}^m),\\
\m{c}^{m}_{r,k}&= \mb{P}^m(\m{x}_{r+1,k}^m) -  \mb{P}(\m{x}_{r+1,k}^m)- \mb{P}^m (\m{x}_{r+1,k-1}^m) + \mb{P}(\m{x}_{r+1,k-1}^m), \\
\m{d}^{m}_{r,k}&=\m{h}^m_{r,k} -\mb{P} (\m{x}_{r+1,k}^m).
\end{aligned}
\end{equation}

One will notice that 
\begin{equation}\label{eqn6:lem:drift:fedmsa:p}
\begin{aligned}
&\mb{E}\left\|\m{h}^m_{r,k}- \mb{P} (\m{x}_{r+1,k}^m)\right\|^2\\
&= 2\mb{E}\left\|\m{a}^{m}_{r,k} \right\|^2+  2  \mb{E}\left\|\m{b}^{m}_{r,k}  + \m{c}^{m}_{r,k} + \m{d}^{m}_{r,k-1}  \right\|^2 \\
& \leq 2\left( \mb{E}\left\|\m{a}^{m}_{r,k} \right\|^2+     \mb{E}\left\|\m{b}^{m}_{r,k} \right\|^2+ \mb{E}\left\|\m{c}^{m}_{r,k} + \m{d}^{m}_{r,k-1} \right\|^2\right)\\
&\leq 2\left( \mb{E}\left\|\m{a}^{m}_{r,k} \right\|^2+     \mb{E}\left\|\m{b}^{m}_{r,k} \right\|^2+ (1+K) \mb{E}\left\|\m{c}^{m}_{r,k} \right\|^2 + \left(1+\frac{1}{K} \right) \mb{E}\left\|\m{d}^{m}_{r,k-1} \right\|^2  \right).
    \end{aligned}
\end{equation}
Here, the last equality uses Lemma~\ref{lem:rand:zer} since  $\mb{E}[\m{b}_{r,k}^m|\mathcal{F}_{r,k}^m]=0$, {by definition}.

From Assumption~\ref{assum:bias}, we have 
\begin{subequations}    
\label{eqn7:lem:drift:fedmsa}
\begin{align}
\left\|\m{b}_{r,k}^m \right\|^2 &\leq  \frac{\bar{L}_p^2}{b} \left\|\m{x}^{m}_{r+1,k}-\m{x}^{m}_{r+1,k-1}\right\|^2=  \frac{\bar{L}_p^2\alpha^2}{b} \left\|\m{h}^{m}_{r+1,k-1}\right\|^2.
    \end{align}
 \begin{align}
\left\|\m{c}_{r,k}^m \right\|^2 &\leq \tau^2_0 \left\|\m{x}^{m}_{r+1,k}-\m{x}^{m}_{r+1,k-1}\right\|^2 \leq   \tau^2_0 \alpha^2 \left\|\m{h}^{m}_{r+1,k-1}\right\|^2.
    \end{align}   
Further, 
\begin{equation}
\begin{aligned}
\left\|\m{d}_{r,k}^m \right\|^2 
  & \leq  \left(1+ \frac{1}{K}\right) \left\|\m{d}_{r,k-1}^m \right\|^2 + \alpha^2\left( \frac{\bar{L}_p^2}{b} +\tau^2_p \left(1+K\right)\right)  \left\|\m{h}^{m}_{r+1,k-1}\right\|^2+  2 \mb{E}\left\|\m{a}^{m}_{r,k} \right\|^2.\\
    \end{aligned}
\end{equation}
\end{subequations}
Let $\kappa_p= K \left( \frac{\bar{L}_p^2}{b} +\tau_0^2 \left(1+K\right)\right) $.  Now, iterating equation \eqref{eqn6:lem:drift:fedmsa:p} and using $  \m{h}^{m}_{r+1,0}=\m{h}_{r}, ~~\forall m\in [M]$, we obtain
\begin{equation}\label{eqn6:lem:drift:fedmsa}
\begin{aligned}
\left\|\m{d}_{r,k}^m \right\|^2& \leq  2 e \left\|\m{d}_{r,0}^m \right\|^2 + 2 e \kappa_p \alpha^2 \frac{1}{K} \sum_{k=1}^K \left\|\m{h}^{m}_{r+1,k-1}\right\|^2+  2\sum_{k=1}^K\mb{E}\left\|\m{a}^{m}_{r,k} \right\|^2\\
 & \leq 2 e \left\|\m{d}_{r,0}^m \right\|^2 + \frac{2 \kappa_p e \alpha^2}{K} \sum_{k=1}^K \left\|\m{d}_{r,k-1}^m\right\|^2 + \frac{2 \kappa_p e \alpha^2 }{K}\sum_{k=1}^K \left\|\mb{P}(\m{x}_{r,k-1}^m)\right\|^2 +  2 \sum_{k=1}^K\mb{E}\left\|\m{a}^{m}_{r,k} \right\|^2.
\end{aligned}
\end{equation}
Averaging this inequality from $k=1$ to $K$ and choosing sufficiently small $\alpha$ such that $2 \kappa_p e \alpha^2 \leq 1/2$. This gives the desired result.  

Similarly, let 
\begin{equation}\label{eqn5:lem:drift:fedmsa}
\begin{aligned}
 \m{b}_{r,k}^{m,n} &=\m{q}^{m,n}(\m{z}_{r+1,k}^{m,n-1},\m{z}_{r+1,k}^{m,n})- \m{q}^{m,n}(\m{z}_{r+1,k-1}^{m,n-1},\m{z}_{r+1,k-1}^{m,n})\\
 &- \mb{S}^{m,n}(\m{z}_{r+1,k }^{m,n-1},\m{z}_{r+1,k }^{m,n})+ \mb{S}^{m,n}(\m{z}_{r+1,k -1}^{m,n-1},\m{z}_{r+1,k -1}^{m,n}),\\
\m{c}^{m,n}_{r,k}&=  \mb{S}^{n}(\m{z}_{r+1,k }^{m,n-1},\m{z}_{r+1,k }^{m,n})  -  \mb{S}^{m,n}(\m{z}_{r+1,k }^{m,n-1},\m{z}_{r+1,k }^{m,n})\\
&-  \mb{S}^{m,n}(\m{z}_{r+1,k-1 }^{m,n-1},\m{z}_{r+1,k-1 }^{m,n})+\mb{S}^{n}(\m{z}_{r+1,k-1 }^{m,n-1},\m{z}_{r+1,k-1 }^{m,n}), \\
\m{d}^{m,n}_{r,k}&= \m{q}_{r,k}^{m,n} - \mb{S}^{n}(\m{z}_{r+1,k}^{m,n-1},\m{z}_{r+1,k}^{m,n}).
\end{aligned}
\end{equation}

\begin{equation}\label{eqn6:lem:drift:fedmsa:in}
\begin{aligned}
\mb{E}\left\|\m{d}^{m,n}_{r,k}\right\|^2 &\leq 2\left(     \mb{E}\left\|\m{b}^{m,n}_{r,k} \right\|^2+ (1+K) \mb{E}\left\|\m{c}^{m,n}_{r,k} \right\|^2 + \left(1+\frac{1}{K} \right) \mb{E}\left\|\m{d}^{m,n}_{r,k-1} \right\|^2  \right),
    \end{aligned}
\end{equation}
From Assumption~\ref{assum:bias}, we have 
\begin{subequations}    
\label{eqn7:lem:drift:fedmsa}
\begin{align}
\left\|\m{b}_{r,k}^m \right\|^2 &\leq  \frac{\bar{L}_{s,n}^2}{b} \left\|\m{z}^{m,n}_{r+1,k}-\m{z}^{m,n}_{r+1,k-1}\right\|^2=  \frac{\bar{L}_{s,n}^2\beta_n^2}{b} \left\| \m{q}_{r,k-1}^{m,n}\right\|^2.
    \end{align}
 \begin{align}
\left\|\m{c}_{r,k}^{m,n} \right\|^2 &\leq  \tau^2_{n}  \left\|\m{z}^{m,n}_{r+1,k}-\m{z}^{m,n}_{r+1,k-1}\right\|^2 \leq \tau^2_{n}   \beta_n^2 \left\|\m{q}^{m,n}_{r,k-1}\right\|^2.
    \end{align}   
\end{subequations}
Let $\kappa_{s,n}= K \left(\bar{L}_{s,n}^2/b +\tau^2_n \left(1+K\right)\right)$. Now, iterating equation \eqref{eqn6:lem:drift:fedmsa:in} and using $  \m{q}^{m}_{r+1,0}=\m{q}_{r}, ~~\forall m\in [M]$, we obtain
\begin{equation}\label{eqn6:lem:drift:fedmsa}
\begin{aligned}
\left\|\m{d}_{r,k}^{m,n} \right\|^2
 & \leq 2 e \left\|\m{d}_{r,0}^{m,n} \right\|^2 + \frac{2 \kappa_{s,n} e \beta_n^2}{K} \sum_{k=1}^K \left\|\m{d}_{r,k-1}^{m,n}\right\|^2 + \frac{2 \kappa_{s,n} e \beta_n^2 }{K}\sum_{k=1}^K \left\|\mb{S}^{n}(\m{z}_{r+1,k-1}^{m,n-1},\m{z}_{r+1,k-1}^{m,n})\right\|^2\\
 & \leq 2 e \left\|\m{d}_{r,0}^{m,n} \right\|^2 + \frac{2 \kappa_{s,n} e \beta_n^2}{K} \sum_{k=1}^K \left\|\m{d}_{r,k-1}^{m,n}\right\|^2 + \frac{2 \kappa_{s,n} e \beta_n^2 L_{\m{z}}}{K}\sum_{k=1}^K \left\|\m{z}_{k-1}^{m,n} - \m{z}_{k-1}^{m,n,*} \right\|^2.
\end{aligned}
\end{equation}
Here, the second inequality holds by Assumption \ref{assum:lip:y*}. Averaging this inequality from $k=1$ to $K$ and choosing sufficiently small $\beta_n$ such that $2 \kappa_{s,n} e \beta^2_n \leq 1/2$. This gives the desired result.  
\end{proof}

\begin{lemma}\label{lem:drift2:fedmsa}
Suppose Assumptions~ \ref{assum:lip:y*}--\ref{assum:heter:h} hold.  Further, assume $K \geq 1$ and $ \alpha\leq  \sqrt{\rho M Kb}/(C''\bar{L}_pK)$ and $ \beta_n\leq  \sqrt{\rho M Kb}/(D''\bar{L}_{s,n}K)$ for some constants $C''$ and $D''$. Then,  Algorithm~\ref{alg:fedmsa} guarantees:
%
\begin{subequations}
\begin{align}\label{eqn1:lem:drift:h}
\frac{1}{R} \sum_{r=0}^{R-1} \mb{E} \left\| \m{h}_r- \mb{P}(\m{x}_r)\right\|^2  \leq  \frac{1}{R} \sum_{r=0}^{R-1}  \mb{E}\|\tilde{\m{Z}}_{r} - \tilde{\m{Z}}^*_r\|^2+ \frac{1}{R} \sum_{r=0}^{R-1} \frac{1}{K} \sum_{k=1}^{K-1}  \mb{E} \left\| \mb{P}(\tilde{\m{x}}_{r+1,k}^{\tilde{m}})\right\|+\frac{E_1(\rho)}{MK} \sigma^2_h, &\\
\label{eqn1:lem:drift:q}
 \frac{1}{R}\sum_{r=0}^{R-1}\mb{E}\left\|\m{q}_r-\mb{S}\left(\m{z}^{n-1}_{r},\m{z}^{n}_{r}\right)\right\| \leq 
\frac{1}{R} \sum_{r=0}^{R-1} \frac{1}{K} \sum_{k=1}^{K-1}  \mb{E} \left\|\mb{S}^{n}(\m{z}_{r+1,k}^{\tilde{m},n-1},\m{z}_{r+1,k}^{\tilde{m},n})\right\|
 +  \frac{E_2(\rho)}{MK} \sigma^2_q,&
\end{align}
where $\rho$ is the momentum parameter and
\begin{align}
  E_1(\rho):= \frac{48 e  \rho}{ b}  +  \frac{24 e }{ \rho R b_0},~~~  E_2(\rho):= \frac{2\rho}{b}+\frac{1}{\rho R b_0}.
\end{align}
\end{subequations}

\end{lemma}
\begin{proof}
It follows from Algorithm~\ref{alg:fedmsa} that 
\begin{equation}
 \m{h}^m_r =  \m{p}^{m} \left(\m{x}_r,\m{Z}_r;\xi^m_r\right) + (1-\rho) \left(\m{h}_{r-1} - \m{p}^{m}\left(\m{x}_{r-1},  \m{Z}_{r-1} ,\xi^m_r\right)\right).
\end{equation}
Note that $\mb{P} (\m{x})=\mb{P} (\m{x}, \m{Z}^\star)$  and $\m{p} (\m{x})=\m{p} (\m{x}, \m{Z}^\star)$. Let
\begin{equation}\label{eqn5:lem:drift:fedmsa}
\begin{aligned}
\m{a}_{r}^m &= (1-\rho)(-\m{p}^{m} \left(\m{x}_{r-1},\m{Z}_{r-1};\xi^m_r\right)+\m{p}^{m} \left(\m{x}_{r-1};\xi^m_r\right)) \\
&+\m{p}^{m} \left(\m{x}_{r},\m{Z}_r;\xi^m_r\right) - \m{p}^{m} \left(\m{x}_{r};\xi^m_r\right),
\\
 \m{b}_{r}^m &=  \m{p}^{m} \left(\m{x}_{r};\xi^m_r\right) -\m{p}^{m} \left(\m{x}_{r-1};\xi^m_r\right)\\
 &- \mb{P}^m(\m{x}_{r}) +\mb{P}^m (\m{x}_{r-1}),\\
\m{c}^{m}_{r}&= \m{p}^{m} \left(\m{x}_{r};\xi^m_r\right)- \mb{P} (\m{x}_{r}), \\
\m{d}_{r}^m&=\m{h}_{r}^m - \mb{P}(\m{x}_{r}).
\end{aligned}
\end{equation}
Then, it follows from the update rule of ALgorithm~\ref{alg:fedmsa} that 
\begin{align*}
\m{d}_{r} =\m{a}_{r}+ (1 -\rho)\m{d}_{r-1}+  \rho \m{c}_{r}+(1-\rho) \m{b}_{r}.
\end{align*}
Hence, 
\begin{align*}
\nonumber
\mb{E}_r\left\|\m{d}_{r}\right\|^2
 &= 2\left\|\m{a}_{r}\right\|^2 
+2(1 -\rho)^2  \left\|\m{d}_{r-1}\right\|^2 + 2\rho^2 \left\| \m{c}_r\right\|^2+2(1-\rho)\left\|
 \m{b}_r\right\|^2\\
 & \leq 2  \|\m{Z}_{r} - \m{Z}^*(\m{x}_{r})\|^2+ 2 (1-\rho)^2\|\m{Z}_{r-1} - \m{Z}^*(\m{x}_{r-1})\|^2\\
 &+ (1 -\rho)^2  \left\| \m{d}_{r-1}\right\| + 2 \rho^2 \frac{\sigma^2}{MKb}+ 2(1-\rho)^2 \frac{\bar{L}_p^2}{MKb} \left\| \m{x}_r - \m{x}_{r-1} \right\|^2.
\end{align*}
Hence,
\begin{align*}
\nonumber
\rho \sum_{r=0}^{R-1}\mb{E}\left\|\m{d}_r\right\|
 &=  \sum_{r=0}^{R-1}\mb{E}\left\|\m{d}_r\right\|-(1-\rho) \sum_{r=0}^{R-1} \left\|\m{d}_{r-1}\right\| \\
 &= \sum_{r=1}^{R}\mb{E}\left\|\m{d}_r\right\|^2-(1-\rho) \sum_{r=0}^{R-1} \left\|\m{d}_{r-1}\right\|^2  - \left\|\m{d}_{R}\right\|^2  +\left\|\m{d}_{0}\right\|^2  \\
 &\leq 2\sum_{r=0}^{R-1} \mb{E}\|\m{Z}_{r} - \m{Z}^*(\m{x}_{r})\|^2+  \frac{2(1-\rho)^2 \bar{L}_p^2}{ MKb} \sum_{r=0}^{R-1}\mb{E} \left\|\m{x}_r-\m{x}_{r-1}\right\|^2 + 2\rho^2 R \frac{\sigma^2}{MKb} + \mb{E}\left\|\m{d}_{0}\right\|^2 \\
 & \leq2 \sum_{r=0}^{R-1} \mb{E}\|\m{Z}_{r} - \m{Z}^*(\m{x}_{r})\|^2+  \frac{ 2(1-\rho)^2\bar{L}_p^2}{ MKb} \sum_{r=0}^{R-1}\mb{E} \left\|\m{x}_r-\m{x}_{r-1}\right\|^2 + 2\rho^2 R \frac{\sigma^2}{MKb} +
\frac{\sigma^2}{ Mb_0}.
\end{align*}
Now, from Lemma~\ref{lem:drift:fedmsa}, we have 
\begin{align*}
\nonumber 
\frac{1}{R}\sum_{r=0}^{R-1}\mb{E}\left\|\m{d}_r\right\| &\leq  \frac{2}{R}\sum_{r=0}^{R-1} \mb{E} \|\m{Z}_{r} - \m{Z}^*(\m{x}_{r})\|^2+ 
\frac{160\bar{L}_p^2 K^2 \alpha^2}{\rho  MKb } \frac{1}{R} \sum_{r=0}^{R-1}\left(\mb{E} \left\| \m{d}_r\right\|^2 + \frac{1}{K} \sum_{k=1}^{K-1}\mb{E} \left\|\mb{P}  (\m{x}_{r+1,k}^{\tilde{m}})\right\|^2 \right)\\
& + \frac{2\rho  \sigma^2}{MKb}+\frac{\sigma^2}{\rho RMKb_0}\\
&\leq \frac{2}{R}\sum_{r=0}^{R-1} \mb{E} \|\m{Z}_{r} - \m{Z}^*(\m{x}_{r})\|^2+\frac{ 1}{ (24 e +1)R} \sum_{r=0}^{R-1} \left(\mb{E} \left\| \m{d}_r\right\|^2 +\frac{1}{K} \sum_{k=1}^{K-1}\mb{E} \left\|\mb{P}  (\m{x}_{r+1,k}^{\tilde{m}})\right\|^2  \right)\\
& +  \frac{2\rho \sigma^2}{MKb}+\frac{\sigma^2}{\rho R MKb_0}.
\end{align*}
Here the last inequality uses our assumption on $\alpha$. 

The proof for \eqref{eqn1:lem:drift:q} follows similarly.
\end{proof}

\subsection{Proof of Theorem~\ref{thm:fedmsa}}
\begin{proof}
Let $f$ be a function with gradient $\mb{P}$.  From  Algorithm~\ref{alg:fedmsa} and the Lipschitz property of $ \mb{P}(\m{x})$, we have
\begin{equation} \label{eqn:thm:decf}
\begin{aligned}
f(\m{x}_{r+1,k+1}^{\tilde{m}}) - f(\m{x}_{r+1,k}^{\tilde{m}})
    & \leq \left\langle \m{x}_{r+1,k+1}^{\tilde{m}} - \m{x}_{r+1,k}^{\tilde{m}}, \mb{P}(\m{x}_{r+1,k}^{\tilde{m}}) \right \rangle + \frac{\alpha^2 L^{\m{h}}}{2} \left\|\m{h}_{r,k}^{\tilde{m}}\right\|^2\\
    & \leq  \left\langle \m{h}_{r+1,k}^{\tilde{m}},  \mb{P}(\m{x}_{r+1,k}^{\tilde{m}}) \right \rangle + \frac{\alpha^2 L^{\m{h}}}{2} \left\|\m{h}_{r,k}^{\tilde{m}}\right\|^2 \\
&\leq -\frac\alpha2\|\mb{P}(\m{x}_{r+1,k}^{\tilde{m}})\|^2 - \frac\alpha2\|\m{h}_{r,k}^{\tilde{m}}\|^2\\
&+ \frac\alpha2 \|\mb{P}(\m{x}_{r+1,k}^{\tilde{m}})-\m{h}_{r,k}^{\tilde{m}}\|^2 + \alpha^2 \frac{L^h}2\|\m{h}_{r,k}^{\tilde{m}}\|^2.
\end{aligned}
\end{equation}
The proof uses the idea of non-FL MSA \cite{shen2022single}. Specifically,  following \cite{shen2022single}, we consider the following Lyapunov function 
\begin{equation}\label{eqn:lyapfunc}
 \mb{W}_{r+1,k}^{\tilde{m}}:= f(\m{x}_{r+1,k}^{\tilde{m}}) + \sum_{n=1}^N \left\|\m{z}_{r+1,k}^{n,\tilde{m}} -\m{z}_{r+1}^{n,*} (\m{z}_{r+1,k}^{n-1,\tilde{m}} )\right\|^2.  
\end{equation}
We bound the difference between two Lyapunov functions. It follows from Lemma~\eqref{lem:inn:diff} and  \eqref{eqn:thm:decf}-\eqref{eqn:lyapfunc} that
\begin{align*}
\nonumber 
 \mb{W}_{r+1,k+1}^{\tilde{m}}- \mb{W}_{r+1,k}^{\tilde{m}} 
  & \leq - A \|\mb{P}(\m{x}_{r+1,k}^{\tilde{m}})\|^2+ B \|\mb{P}(\m{x}_{r+1,k}^{\tilde{m}})-\m{h}_{r,k}^{\tilde{m}}\|^2
  \label{eqn:lyap1}
  \\
  \nonumber
 & + C\sum_{n=1}^N  \left\| \m{z}_{r+1,k}^{\tilde{m},n}- \m{z}^{n,*} (\m{z}_{r+1,k}^{\tilde{m},n-1})\right\|^2\\
 &+ D \sum_{n=1}^N  \left\| \m{q}^{\tilde{m},n}_{r+1,k}-\mb{S}^n(\m{z}_{r+1,k}^{\tilde{m},n-1},\m{z}_{r+1,k}^{\tilde{m},n})\right\|^2,
\end{align*}
 where
 $A:= \alpha\left(-1/2+2L_{\m{z},1}^2 \alpha\right)$, ~and $B:= 2/2\alpha (1+ L_{\m{z},1}^2 \alpha)$.

Note that $\m{x}^{\tilde{m}}_{r+1,1} =\m{x}_r$ and $\m{x}^{\tilde{m}}_{r+1,k+1} =\m{x}_{r+1}$. Hence, taking expectation and using Lemma~\ref{lem:drift:fedmsa2}, we obtain
\begin{equation}
\begin{aligned}
  \qquad \frac{1}{K} \left(\mb{W}_{r+1}^{\tilde{m}}- \mb{W}_{r}^{\tilde{m}}\right)
 & \leq - \frac{A (\alpha)+B(\alpha)}{K}\sum_{k=1}^K \mb{E}\|\mb{P}(\m{x}_{r+1,k}^{\tilde{m}})\|^2\\
 &+  B(\alpha) \mb{E}\left\|\m{h}_{r}- \mb{P} (\m{x}_{r})\right\|^2 
  \\
 &+   D (\alpha) \sum_{n=1}^N 
  \mb{E}\left\|\m{q}_{r}^{n} -\mb{S}^{n}(\m{z}_{r}^{n-1},\m{z}_{r}^{n})\right\|^2\\
   & +\frac{C (\alpha)+B(\alpha) L_{\m{h},\m{z}}K+ D (\alpha)}{K}  \sum_{k=1}^K \sum_{n=1}^N  \mb{E}\left\| \m{z}_{r+1,k}^{\tilde{m},n}- \m{z}^{n,*} (\m{z}_{r+1,k}^{\tilde{m},n-1})\right\|^2.
 \label{eqn:lyap12}
\end{aligned}
\end{equation}
Here the inequality uses \eqref{lem:drift:fedmsa}.

Then, it follows from Lemma~\ref{lem:drift2:fedmsa} that
\begin{subequations}\label{eqn:lyap1}
\begin{align}
  \qquad \frac{1}{K} \left(\mb{W}_{r+1}^{\tilde{m}}- \mb{W}_{r+1}^{\tilde{m}}\right)
 & \leq - \frac{A (\alpha)+3/2*B(\alpha)}{K}\sum_{k=1}^K\|\mb{P}(\tilde{\m{x}})\|^2\\
   &+\frac{ C (\alpha)+B(\alpha) L_{\m{h},\m{z}}K+ D (\alpha)  +D (\alpha)}{K}\|\m{Z}_{r} - \m{Z}^*_r\|^2 \\
&+  \frac{B(\alpha)   E_1(\rho) \sigma^2_h+  D (\alpha)  E_2(\rho)}{K} \sigma^2_q.
 \label{eqn:lyap12}
\end{align}
\end{subequations}
Consider the following choices of stepsize 
\begin{align}\label{eqn:al:cho}
\nonumber 
    - A (\alpha)+3/2B(\alpha)&\leq \frac{1}{8} \alpha, \\
    \nonumber 
     C (\alpha)+B(\alpha) L_{\m{h},\m{z}}K+ D (\alpha)  +D (\alpha) &\leq -\lambda \alpha,      
     \\
    \max \left( B(\alpha),  D (\alpha)\right) &\leq \alpha. 
\end{align}
Averaging \eqref{eqn:lyap1} from $\kappa =0, \ldots, R-1$ and using \eqref{eqn:al:cho}, we get 
\begin{align}\label{eq3:thm:msa}
\|\mb{P}(\tilde{\m{x}})\|^2 + \lambda \|\tilde{\m{Z}}_{r} - \tilde{\m{Z}}^*_r\|^2 \leq  \mc{O} \left(\frac{4}{KR\alpha}(\mb{W}_{0} - \mb{W}_{R})+  \left(\frac{48 e  \rho}{ MKb}  +  \frac{24 e }{\ \rho MKb_0}\right) \sigma^2\right).
\end{align}
Let $b_0= KR$, $\Delta_W=\mb{W}_{0} - \mb{W}_{R}$, $\rho_1 =1/R$, $\rho_2 = \left((\Delta_W L_{\m{h}})^{2/3} (MKb)^{1/3}\right)/ (\sigma^{4/3} R^{2/3})$, and $\rho := \max(\rho_1, \rho_2)$ .  Using the conditions on the stepsize, we obtain 
\begin{align}
\nonumber 
\|\mb{P}(\tilde{\m{x}})\|^2 + \lambda \|\tilde{\m{Z}}_{r} - \tilde{\m{Z}}^*_r\|^2  & \leq \mc{O} \Big(\frac{\tau \Delta_W}{R}+ \frac{\Delta_W }{KR} + \frac{ \Delta_W }{ R\sqrt{Kb}} + \frac{\Delta_W }{R \sqrt{\rho_2 MKb}}\\
\nonumber 
&+ (\rho_1+\rho_2)  \frac{\sigma^2}{MKb} +  \frac{\sigma^2}{\rho RMKb_0} \Big)\\
& \leq \mc{O} \left(\frac{\tau}{R} +\frac{1}{\sqrt{K}R} + \frac{\sigma^2}{MKR} + \left(\frac{\sigma}{MKR}\right)^{2/3}\right).
\end{align}

\end{proof}

\section{Federated Bilevel Optimization}
Recall the federated bilevel  optimization problem
\begin{subequations}\label{fedblo:prob:sup}
\begin{align}
\begin{array}{ll}
\underset{\m{x} \in \mb{R}^{{d}_1}}{\min} &
\begin{array}{c}
f(\m{x}):=\frac{1}{M} \sum_{m=1}^{M} f^m\left(\m{x},\m{w}^*(\m{x})\right) 
\end{array}\\
\text{~s.t.} & \begin{array}[t]{l} \m{w}^*(\m{\m{x}})
\in \underset{ \m{w}\in \mb{R}^{{d}_2}}{\textnormal{argmin}}~~g\left(\m{x},\m{w}\right):=\frac{1}{M}\sum_{m=1}^{M} g^m\left(\m{x},\m{w}\right). 
\end{array}
\end{array}
\end{align}
\end{subequations}
We use a stochastic oracle model, where access to local functions $(f_m, g_m)$ is obtained through stochastic sampling:
\begin{align*}
    \nonumber 
    f^m(\m{x},\m{w}) := \mb{E}_{\xi \sim \mc{A}^m}\left[f^m(\m{x}, \m{w}; \xi)\right],~~
    g^m(\m{x},\m{w}) := \mb{E}_{\zeta \sim \mc{B}^m}\left[g^m(\m{x}, \m{w}; \zeta)\right],
\end{align*}
where $(\xi, \zeta)\sim(\mc{A}^m, \mc{B}^m)$ are stochastic samples at the $m^{\text{th}}$ client. 

We make the following assumptions on the local inner and  outer objectives $(f^m, g^m)$. 
\begin{assumption}\label{assu:fedblo:sup}
 For any $m \in [M]$:
\begin{enumerate}[label={\textnormal{\textbf{D\arabic*.}}}, wide, labelwidth=!, labelindent=0pt]
    \item\label{item:strong:convex} For and $\m{x} \in \mb{R}^{d_1}$, $g^m(\m{x},\cdot)$ is strongly convex with modulus $\lambda_1 > 0$.
    \item\label{item:g:lip2} There exist constants $L_{\m{xw}}, l_{\m{xw}}, l_{\m{ww}}$ such that $\nabla_\m{w} g^m(\m{x},\m{w})$ is $L_{\m{xw}}$-Lipschitz continuous w.r.t. $\m{x}$; $\nabla_\m{w} g^m(\m{x},\m{w})$ is $L_g$-Lipschitz continuous w.r.t. $\m{w}$. $\nabla_{\m{xw}}g^m(\m{x},\m{w})$, $\nabla_{\m{ww}} g^m(\m{x},\m{w})$ are respectively $l_{\m{xw}}$-Lipschitz and $l_{\m{ww}}$-Lipschitz continuous w.r.t. $(\m{x},\m{w})$. These functions have mean Lipschitz continuous property with constants $\bar{L}_{\m{xw}}, \bar{l}_{\m{xw}}, \bar{l}_{\m{ww}}, \bar{L}_{\m{xw}}, \bar{L}_g, \bar{l}_{\m{xw}}, \bar{l}_{\m{ww}}$.
    \item\label{item:f:lip} There exist constants $l_{f\m{x}}, l_{f\m{w}}, l_{f\m{w}}', l_\m{w}$ such that $\nabla_\m{x} f^m(\m{x},\m{w})$ and $\nabla_\m{w} f^m(\m{x},\m{w})$ are respectively $l_{f\m{x}}$ and $l_{f\m{w}}$ Lipschitz continuous w.r.t. $\m{y}$; $\nabla_{w} f^m(\m{x},\m{w})$ is $l_{f\m{w}}'$-Lipschitz continuous w.r.t. $\m{x}$; $f^m(\m{x},\m{w})$ is $l_\m{w}$-Lipschitz continuous w.r.t. $\m{w}$. These functions have mean Lipschitz continuous property with constants $\bar{l}_{f\m{x}}, \bar{l}_{f\m{w}}, \bar{l}_{f\m{w}}', \bar{l}_\m{w}, \bar{l}_{f\m{x}}, \bar{l}_\m{w}$.
\end{enumerate}
\end{assumption}

\begin{assumption}[Bias and variance]\label{assum:bias:blo} 
For all $ m \in  [M]$
\begin{align*}
    \mb{E}_{\xi \sim \mc A^{m}} \left\| \m{p}^{m} \left(\m{x},\m{Z}; \xi\right) - {\mb{P}}^{m}\left(\m{x}, \m{Z}\right) \right\|^2 & \leq \sigma^2,\\
    \mb{E}_{\zeta \sim \mc B^{m}} \left\|
    \m{s}^{m}(\m{x},\m{z}; \zeta) - {\mb{S}}^{m}(\m{x},\m{z}) \right\|^2 & \leq \sigma_1^2
\end{align*}
\end{assumption}
\begin{assumption}[Heterogeneity]
\label{assum:heter:h:blo}
For all $ m \in  [M]$, the set of mappings $ \{\mb{P}^{m}\}$ and $ \{\mb{S}^{m}\}$ are $\tau_0$ and $\tau_1$--Heterogeneous, respectively. 
\end{assumption}

\begin{lemma}\label{prop:cond:bilevel}
Assumptions~\ref{item:strong:convex}--\ref{item:f:lip},  and \ref{assum:bias:blo} imply Assumptions~\ref{assum:lip:y*}--\ref{assum:bias}.
\end{lemma}
\begin{proof}
The proof follows a similar structure to that presented in \cite[Lemma~1]{shen2022single}.
\end{proof}

\begin{corollary}[\textbf{Bilevel}]\label{thm:fedmsa:bilevel}
 Suppose Assumptions~\ref{item:strong:convex}--\ref{item:f:lip}, \ref{assum:bias:blo}, and \ref{assum:heter:h:blo} hold for the gradient mappings defined in \eqref{eq:map:bilevel}.
%
%
Further, assume $\alpha=\mc{O}(\frac{1}{\tau K})$, and $\beta=\mc{O}(\frac{1}{\tau K})$, $\rho=\Theta(\frac{1}{R})$, then
\begin{align*}
\E \left\|\nabla\m{f}(\tilde{\m{x}})\right\|^2 +\E \left\|\tilde{\m{w}}-\m{w}^{*}(\tilde{\m{x}})\right\|^2   \leq  \mc{O}\left(\frac{\tau}{R} +\frac{1}{\sqrt{K}R} + \frac{\sigma^2}{MKR} + \left(\frac{\sigma}{MKR}\right)^{2/3}\right)
\end{align*}
for $\sigma:= \max (\sigma_0, \sigma_1)$ and $\tau=\max(\tau_0, \tau_1)$. 
\end{corollary}

\textbf{Comparison with previous federated/non-federated BLO results}:
Theorem~\ref{thm:fedmsa:bilevel} shows that the communication complexity of 
\textsc{FedBLO} to obtain the $\epsilon$-stationary point is upper bounded by  $\mc{O} ( \tau \epsilon^{-1})$. The key feature of \textsc{FedBLO} is updating the local indirect gradient $- \nabla^2_{\m{x}\m{w}} g(\m{x},\m{w}^*(\m{x}))\m{v}^{m,*}(\m{x})$. The prior methods \cite{tarzanagh2022fednest,xiao2022alternating,huang2022fast} was unable to update this indirect gradient in their federated setting. Further,  our algorithm not only achieves linear speedup similar to \cite{huang2023achieving}, but also significantly enhances the convergence rate and sample complexity of federated bilevel methods \cite{tarzanagh2022fednest,xiao2022alternating,huang2022fast}. For detailed comparisons, please refer to Table~\ref{table:fedmsa:results}.

\section{Federated Compositional Optimization}
In this section, we consider the federated multi-level compositional optimization problem 
\begin{equation*}
    \min_{\m{x}\in \mathbb{R}^{d_0}}f(\m{x}) :=  \frac{1}{M}\sum_{i=1}^M f^{N}_i( \sum_{i=1}^M f^{N-1}_i(\dots \sum_{i=1}^M f^0_i(\m{x})\dots).     
\end{equation*}
where $f^{m,n} : \mathbb{R}^{d_n} \mapsto \mathbb{R}^{d_{n+1}}$ for $m \in [M]$, $n=0,1,\ldots,N$ with $d_{N+1}=1$. Only stochastic evaluations of each layer function are accessible, i.e.,
\begin{equation*}
    f^{m,n} (\m{x}) := \E_{\zeta^{m,n}}[f^{m,n} (\m{x};\zeta^{m,n})],~ m \in [M], n=0,1,\ldots,N.
\end{equation*}
where $\{ \zeta^{m,n} \}_{m,n}$ are random variables.  Here, we slightly overload the notation and use $f^{m,n} (\m{x};\zeta^{m,n})$ to represent the stochastic version of the mapping.

\begin{assumption}
For any $(m,n) \in [M]\times (N]$:
\label{item:f:sc}  $f^{m,n}(\cdot)$ is $L_{\m{z},n}$-Lipschitz continuous and $L_{\m{z},n}'$-smooth. These functions have mean Lipschitz property with $\bar{L}_{\m{z},n}$ and $\bar{L}_{\m{z},n}'$, respectively.
\end{assumption}
\begin{assumption}\label{assum:heter:mf}
The set of functions $\{{f}^{n,m}\}$ are  $\tau_n$--Heterogeneous for $n=0, \ldots, N$.
\end{assumption}
\begin{assumption}[Bias and variance]\label{assum:bias:comp}  
 $f^{m,n} (\m{z}^{n-1};\zeta^n)$ and $\nabla f^{m,n} (\m{z}^{n-1};\zeta^n)$ are respectively the unbiased estimators of $f^{m,n} (\m{z}^{n-1})$ and $\nabla f^{m,n}(\m{z}^{n-1})$ with bounded variance $\sigma_n$; $ f^{m,0} (\m{x};\hat{\zeta}^0)$ and $\nabla f^{m,0} (\m{x};\hat{\zeta}^0)$ are respectively the unbiased estimators of $f^{m,0} (\m{x})$ and $\nabla f^{m,0}(\m{x})$ with bounded variance $\sigma_0$.

\end{assumption}

\begin{lemma}\label{prop:cond:bilevel}
Assumptions~\ref{item:f:sc}, \ref{assum:bias:comp}, and \ref{assum:heter:mf} imply Assumptions~\ref{assum:lip:y*}--\ref{assum:heter:h}.
\end{lemma}
\begin{proof}
The proof follows a similar structure to that presented in \cite[Lemma~3]{shen2022single}.
\end{proof}

\begin{corollary}[\textbf{Multi-Level Compositional}]\label{thm:fedmsa:mco}
 Suppose Assumptions~\ref{item:f:sc}, \ref{assum:bias:comp}, and \ref{assum:heter:mf} hold. Further, assume $\alpha=\mc{O}(\frac{1}{\tau K})$, and $\beta=\mc{O}(\frac{1}{\tau K})$, $\rho=\mc{O}(\frac{1}{R})$, then 
\begin{align*}
\E \left\|\nabla\m{f}(\tilde{\m{x}})\right\|^2 + \sum_{n=1}^N\E \left\|\tilde{\m{z}}^n-\m{z}^{n,*}(\tilde{\m{z}}^{n-1})\right\|^2 \leq \mc{O}\left(\frac{\tau}{R} +\frac{1}{\sqrt{K}R} + \frac{\sigma^2}{MKR} + \left(\frac{\sigma}{MKR}\right)^{2/3}\right).
\end{align*}
Here, $\sigma:= \max (\sigma_0, \sigma_1,  \cdots, \sigma_N)$,  $\tau=\max(\tau_0, \cdots, \tau_N)$,  
and $\mc{O}$ hides problem dependent constants of a polynomial of $N$.
\end{corollary}

\textbf{Comparison with previous federated/non-federated MCO results}: Corollary \ref{thm:fedmsa:mco} implies a sample complexity of $\mc{O} ( \epsilon^{-1.5})$ and communication complexity of $\mc{O} ( \tau \epsilon^{-1})$. The $\tau$ factor controls the benefit we can obtain from small heterogeneity.  Moreover,  existing federated methods for MCO only consider the double sequence case, i.e., $N=1$ \cite{huang2021compositional,tarzanagh2022fednest,huang2022faster}. Our approach also provides the first communication complexity analysis; please refer to Table~\ref{table:fedmsa:results} for further comparisons.

\section{Other Technical Lemmas}\label{sec:techn}
 We collect additional technical lemmas in this section.

\begin{lemma}\label{lem:Jens}
For any set of vectors $\{\m{x}_i\}_{i=1}^m$ with $\m{x}_i\in\mathbb{R}^d$, we have
 \begin{equation}
     \norm[\bigg]{ \sum\limits_{i=1}^{m} \m{x}_i}^2 \leq m \sum\limits_{i=1}^{m} {\Vert \m{x}_i \Vert}^2.
\label{eqn:Jens}
 \end{equation} 
 \end{lemma}
\begin{lemma}\label{lem:trig}
For any $\m{x},\m{y}\in\mathbb{R}^d$, the following holds for any $c >0$:
\begin{equation}
    {\Vert \m{x}+\m{y} \Vert}^2 \leq (1+c){\Vert \m{x} \Vert}^2 + \left(1+\frac{1}{c}\right){\Vert \m{y} \Vert}^2.
\label{eqn:triang}
\end{equation}
\end{lemma}
\begin{lemma}\label{lem:rand:zer}
For any set of independent, mean zero random variables $\{\m{x}_i\}_{i=1}^m$ with $\m{x}_i\in\mathbb{R}^d$, we have
 \begin{equation}
     \mb{E}\left[\norm[\bigg]{ \sum\limits_{i=1}^{m} \m{x}_i}^2\right] = \sum\limits_{i=1}^{m} \mb{E}\left[{\Vert \m{x}_i \Vert}^2\right].
\label{eqn:randJens}
 \end{equation} 
\end{lemma}
\end{document}